\documentclass[10pt,journal,compsoc]{IEEEtran}
\ifCLASSOPTIONcompsoc
  \usepackage[nocompress]{cite}
\else
  \usepackage{cite}
\fi

\ifCLASSINFOpdf
  \usepackage[pdftex]{graphicx}
\else
\fi
\usepackage{hyperref}
\hypersetup{hidelinks}
\usepackage{amsmath}
\usepackage{amsthm}
\newtheorem{theorem}{\indent Theorem}

\usepackage[ruled,norelsize]{algorithm2e}
\makeatletter
\newcommand{\removelatexerror}{\let\@latex@error\@gobble}
\makeatother

\usepackage{array}

\ifCLASSOPTIONcompsoc
 \usepackage[caption=false,font=footnotesize,labelfont=sf,textfont=sf]{subfig}
\else
 \usepackage[caption=false,font=footnotesize]{subfig}
\fi

\begin{document}
%
\title{Multi-Agent Collaborative Inference via DNN Decoupling: Intermediate Feature Compression and Edge Learning}

%
%
%
%

\author{Zhiwei~Hao,
        Guanyu~Xu,
        Yong~Luo,
        Han~Hu,
        Jianping~An,
        and~Shiwen~Mao
\IEEEcompsocitemizethanks{
  \IEEEcompsocthanksitem Zhiwei Hao, Guanyu Xu, Han Hu, and Jianping An are with the School of Information and Electronics, Beijing Institute of Technology, Beijing 100081, China. 
  E-mail: \{haozhw, xuguanyu, hhu, an\}@bit.edu.cn. 
  \IEEEcompsocthanksitem Yong Luo is with School of Computer Science, Wuhan University, Wuhan 430072, China. 
  E-mail: luoyong@whu.edu.cn.
  \IEEEcompsocthanksitem Shiwen Mao is with Department of Electrical and Computer Engineering, Auburn University, Auburn, AL 36849, USA.
  E-mail: smao@auburn.edu.}
}

\IEEEtitleabstractindextext{%
\begin{abstract}
  Recently, deploying deep neural network (DNN) models via collaborative inference, which splits a pre-trained model into two parts and executes them on user equipment (UE) and edge server respectively, becomes attractive.
  However, the large intermediate feature of DNN impedes flexible decoupling, and existing approaches either focus on the single UE scenario or simply define tasks considering the required CPU cycles, but ignore the indivisibility of a single DNN layer.
  In this paper, we study the multi-agent collaborative inference scenario, where a single edge server coordinates the inference of multiple UEs.
  Our goal is to achieve fast and energy-efficient inference for all UEs.
  To achieve this goal, we design a lightweight autoencoder-based method to compress the large intermediate feature at first.
  Then we define tasks according to the inference overhead of DNNs and formulate the problem as a Markov decision process (MDP).
  Finally, we propose a multi-agent hybrid proximal policy optimization (MAHPPO) algorithm to solve the optimization problem with a hybrid action space.
  We conduct extensive experiments with different types of networks, and the results show that our method can reduce up to 56\% of inference latency and save up to 72\% of energy consumption.
\end{abstract}

}

\maketitle

\IEEEdisplaynontitleabstractindextext

%
\IEEEpeerreviewmaketitle

\IEEEraisesectionheading{\section{Introduction}\label{sec:introduction}}


\IEEEPARstart{R}{ecent} years have witnessed a rapid development of intelligent internet of things (IoT) for deep learning (DL)-based applications, e.g., intelligent personal assistant and healthcare applications.
However, the inference of most DL models has an enormous amount of overhead, and this often results in an unacceptable latency or energy consumption, especially on resource-limited IoT UEs.
For example, the ResNet50 model \cite{DBLP:conf/cvpr/HeZRS16} requires 103MB memory and 4 GFLOPs for inference, while a Raspberry Pi-3B device only has 1GB RAM and 3.62 GFLOPS \cite{flops}.
To achieve effective inference, models are usually deployed in a mobile-edge computing (MEC) manner \cite{MEC}.
Specifically, a pre-trained model is decoupled at a proper point, and each pair is deployed on an UE and an edge server respectively.
When receiving an inference task, the front part of the model is first executed on the UE.
Then the obtained intermediate feature is offloaded to the edge server, which will complete the remaining inference task and finally return the result to the UE.
This procedure is called \textit{collaborative inference}.
Compared with cloud computing approaches, collaborative inference avoids direct uploading the input data to the cloud, which is the performance bottleneck of the current DL-based techniques \cite{DBLP:conf/asplos/KangHGRMMT17}.

In collaborative inference, the size of intermediate features should be smaller than the original input.
Otherwise, it's better to offload the original input data to the edge server.
However, the size of intermediate features of most DNN models are usually larger than the original input.
Furthermore, in common MEC scenarios, there usually exists multiple UEs, and their interference on the offloading channel would incur additional latency.
Hence, achieving effective multi-agent collaborative inference is challenging.

To address this problem, some existing works propose to compress the intermediate features of DNNs to reduce the transmission overhead \cite{DBLP:conf/islped/EshratifarEP19, DBLP:conf/icc/ShaoZ20}, where DNN together with quantization or entropy coding are utilized to build the compressor.
However, the architectures of such compressors are usually complex, which lead to a high inference latency under high uplink rate conditions.

In the real world, it is common for a single edge server to serve multiple UEs, where the UEs communicate with the server via a shared channel.
The interference between UEs in the offloading channel could incur significant additional latency in this scenario, even with compressed features.
To tackle this issue, some existing methods optimize the offloading decision and channel resource allocation to improve edge task offloading.
The resulting optimization problem is NP-hard \cite{DBLP:journals/ton/ChenJLF16}, and hence deep reinforcement learning (DRL) has been adopted to find the solution.
For example, Li \textit{et al.} \cite{DBLP:conf/wcnc/Li0LL18} use deep Q-learning (DQL) to allocate offloading decision and resources for multiple UEs, and Nath \textit{et al.} \cite{DBLP:conf/icc/NathLWF20} address this problem by employing deep deterministic policy gradient (DDPG).
However, these approaches simply define tasks using the size of input data and the required number of CPU cycles, while the indivisibility of a single DNN layer is ignored.
Furthermore, the multi-agent collaborative inference problem contains a hybrid action space, e.g., the discrete offloading channel and the continuous transmit power, while the existing methods only focus on either a discrete or continuous action space.

To remedy these drawbacks, we propose a DRL framework for multi-agent collaborative inference.
The framework contains a lightweight autoencoder-based intermediate feature compressor and a multi-agent hybrid proximal policy optimization algorithm, named MAHPPO, to train the DRL agent.
At first, to achieve fast and effective intermediate feature compression, we design a lightweight autoencoder-based method, which utilizes autoencoder and quantization to build the compressor.
Then for the multi-agent scenario, we redefine and formulate the problem by taking layer indivisibility into consideration.
Finally, we design an MAHPPO algorithm to solve the problem with a hybrid action space.
In the experiments, we first train the autoencoder-based compressors at each potential partitioning point of the ResNet18 model \cite{DBLP:conf/cvpr/HeZRS16}.
Then we measure the inference overhead of each module on an edge hardware.
Given the collected data, we train a DRL agent via MAHPPO to provide offloading decisions.
The results show that the autoencoder architecture can achieve a high compression rate with little overhead, and the MAHPPO algorithm can reduce up to 56\% of the inference latency and save up to 72\% of energy consumption compared with the full local inference strategy.
By simply tuning a balancing hyperparameter, we can achieve a trade-off between inference latency and energy consumption.
Moreover, we conduct experiments on VGG11 \cite{DBLP:journals/corr/SimonyanZ14a} and MobileNetV2 \cite{DBLP:conf/cvpr/SandlerHZZC18} to further verify the effectiveness of our approach.
Our main contributions are summarized as follows:
\begin{itemize}
  \item We design a lightweight autoencoder-based intermediate feature compression module, which can greatly reduce the transmission overhead while consuming little time and energy.
  \item We measure the latency and energy consumption of the DNN inference task on an edge hardware, and use these measurements on real devices to build model. This is different from the existing works that define tasks using only the input data size and the required amount of CPU cycles, which may not be appropriate for describing real-world applications.
  \item We redefine and formulate the problem of multi-agent collaborative inference as an MDP to enable feasible optimization, where the interference among multiple UEs is taken into consideration.
  \item We propose an MAHPPO algorithm to solve the optimization problem, which can deal with multiple agents and work with the hybrid action space.
\end{itemize}

The remainder of this paper is organized as follows.
We present the autoencoder-based intermediate feature compressor in Sec. \ref{sec:ae}.
The overall system and general formulation is presented in Sec. \ref{sec:model}, which is then reformulated as an MDP to facilitate optimization in Sec. \ref{sec:mdp}.
The proposed MAHPPO algorithm is depicted in Sec. \ref{sec:mahppo}, and Sec. \ref{sec:exp} includes some experiments and analysis.
Finally, we discuss the related works in Sec. \ref{sec:related_works} and conclude our paper in Sec. \ref{sec:conclusion}.

\section[Autoencoder-based Intermediate Feature Compression]{Autoencoder-based Intermediate\\ Feature Compression}
\label{sec:ae}

In this section, we propose a lightweight autoencoder-based intermediate feature compression method.
We first present an overview of the proposed method and then provide the details of each module.
Finally, we present the optimization strategy for model training.

\subsection{Overview}

Existing compressors usually require a large overhead and incur a high latency. To achieve efficient intermediate feature compression, we propose a lightweight autoencoder-based compressing method.
In particular, we adopt a single-layer autoencoder coupled with a quantization module as the compressor.
Autoencoder is an unsupervised learner composed of an encoder and a decoder, and the two parts constitute our compressor and decompressor respectively.
The quantization module represents the feature map using fewer bits to further compress the encoder outputs.

Fig. \ref{fig:compression_overview} illustrates the workflow of the collaborative inference equipped with the feature compression method.
The DNN model is split into two parts, which are deployed on the UE and the edge server, respectively.
When a new inference task arrives, the UE first executes the front part of the model and performs compression of the obtained intermediate feature, which is then transmitted to the edge server via a wireless channel.
On the edge server, the received compressed feature is restored and passed into the remaining part of the original DNN model.
When the inference is completed, the edge server returns the results to the UE.

\subsection{Channel Reduction and Restoration}

\begin{figure}
  \centering
  \includegraphics[width=0.475\textwidth]{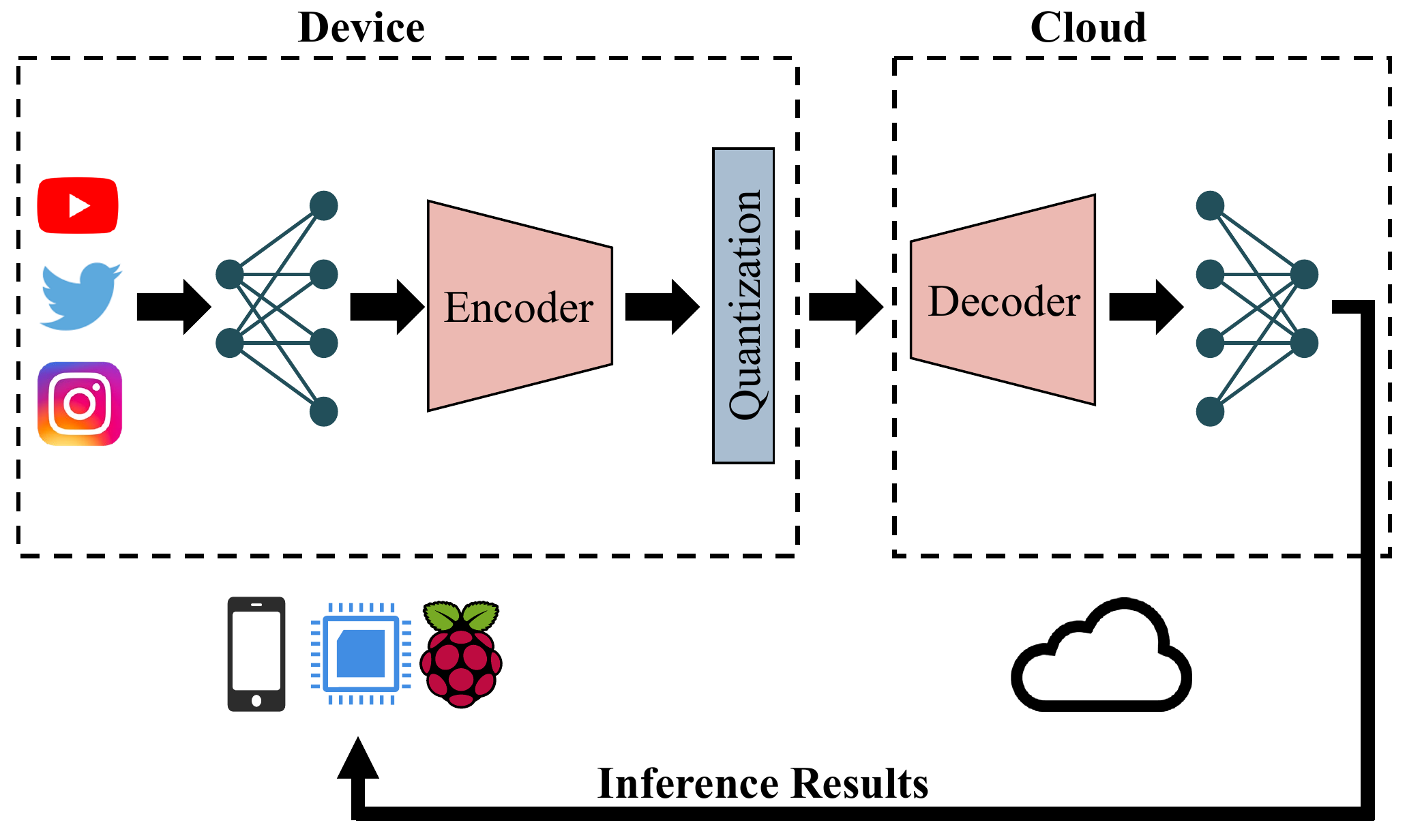}
  \caption{Architecture of the lightweight autoencoder-based intermediate feature compression framework. The DNN model is partitioned into two parts, which are deployed on the UE and the edge server respectively. During inference, the input data go through the front part of the model, the compressor, the wireless channel, the decompressor, and the remaining part of the model successively.}
  \label{fig:compression_overview}
\end{figure}

Intermediate features of DNN usually contain large information redundancy.
This motivates us to design a compressor to learn a compact feature representation.

In order to achieve effective feature compression with limited overhead, both the encoder and decoder are composed of only a convolution layer with a $1\times 1$ kernel.
Suppose the shape of the intermediate feature is $(bs, ch, w, h)$, where each dimension denotes batch size, channel number, width, and height of the feature respectively. 
The convolution layer with kernel size $(ch,ch',1,1)$ shrinks the channel number from $ch$ to $ch'$, where $ch<ch'$.
The compression rate of the channel reduction encoder is defined as $R_c=ch/ch'$, and the corresponding convolution kernel shape in the decoder is $(ch',ch,1,1)$.

When the model serves UEs, there may be inputs beyond the original training dataset, which usually causes a performance degradation.
Nevertheless, the lightweight autoencoder is focused on feature compression at the channel level, and is insensitive to small input domain shifts.
Hence, the proposed autoencoder is an effective and robust feature compressor.

\subsection{Quantization and Dequantization}

Existing works have shown that representing intermediate features using lower bit-width has little impact on inference accuracy \cite{DBLP:conf/icpads/LiHJWWZ18}. 
This motivates us to further compress the output features of the encoder by adopting the quantization technique, which maps floating-point values in feature maps to a smaller set consisting of discrete integers.

On the UE, the quantization procedure can be formulated as:
\begin{equation}
  \mathbf y_i = \text{round}\left(\frac{(2^{c_q}-1)[\mathbf x_i - \min(\mathbf x)]}{\max(\mathbf x)-\min(\mathbf x)}\right),
\end{equation}
where $\mathbf x$ is the intermediate feature to be quantized, $\mathbf x_i$ is the $i$-th value in $\mathbf x$, $y_i$ is output integer of $\mathbf x_i$, and $c_q$ is the bit-width used for quantization.
The maximum and minimum value of $\mathbf x_i$ can be replaced by the result computed on a pre-collected set of feature maps.

On the edge server, the quantized value $y_i$ can be recovered approximately by the dequantization procedure:
\begin{equation}
  \mathbf x'_i = \frac{\mathbf y_i [\max(\mathbf x)-\min(\mathbf x)]}{2^{c_q}-1}+\min(\mathbf x),
\end{equation}
where $\mathbf x'_i$ is the recovered value. 
The rounding operation in the quantization causes round-off error, so $\mathbf x'_i$ is usually not precisely equal to $\mathbf x_i$.
Since the original intermediate features are represented by 32 bit-width floating-point number, the compression rate of the quantization procedure is $R_q=32/c_q$. 
Hence, the overall compression rate of our method is given by:
\begin{equation}
  R=R_c\times R_q=\frac{ch\times 32}{ch'\times c_q}.
\end{equation}

\subsection{Optimization}

The feature compression component is optimized using a two-stage training strategy.
For a given pre-trained model $\mathcal M$ and a selected partitioning point, we first train the autoencoder by minimizing the $l_2$ distance between the original and recovered features.
Moreover, a cross-entropy loss is introduced to minimize the prediction error.
For a sample $x$ in the training set $\mathcal X$, we denote the original feature and the recovered feature of the autoencoder as $\mathbf T_i^x$ and $\mathbf T_o^x$, respectively.
Then the loss function for training the autoencoder can be formulated as:
\begin{equation}
  \mathcal{L}_{ae} = ||\mathbf T_i^x-\mathbf T_o^x||_2+\xi d_{ce}(\mathcal M(x),y),
\end{equation}
where $\xi$ is a balancing hyperparameter, $d_{ce}$ is the cross entropy measurement, and $y$ is the corresponding label of sample $x$.

In the training of the autoencoder, we freeze parameters in the pre-trained model.
When this first-stage training is done, we fine-tune all the parameters in the model on the training set using a small learning rate.

\textbf{Discussion.} The proposed autoencoder-based compressor is specially designed for the MEC scenario, which reduces the size of intermediate features while consuming little time during inference.
The autoencoder training also takes little time and energy due to its lightweight nature.
In practice, we can further improve the compression rate by designing a more complex autoencoder structure or adopting other approaches such as entropy coding.
However, the compressor is designed to reduce the latency caused by data offloading. 
The extra latency introduced by the compressor should be less than the reduced offloading latency.
By adopting the proposed compressor, we can significantly reduce the inference latency and save plenty of energy with little training and inferring overhead on the autoencoder.

\section{Multi-user Collaborative Inference}
\label{sec:model}

\begin{figure}[!t]
  \centering
  \subfloat[]{\includegraphics[width=0.475\textwidth]{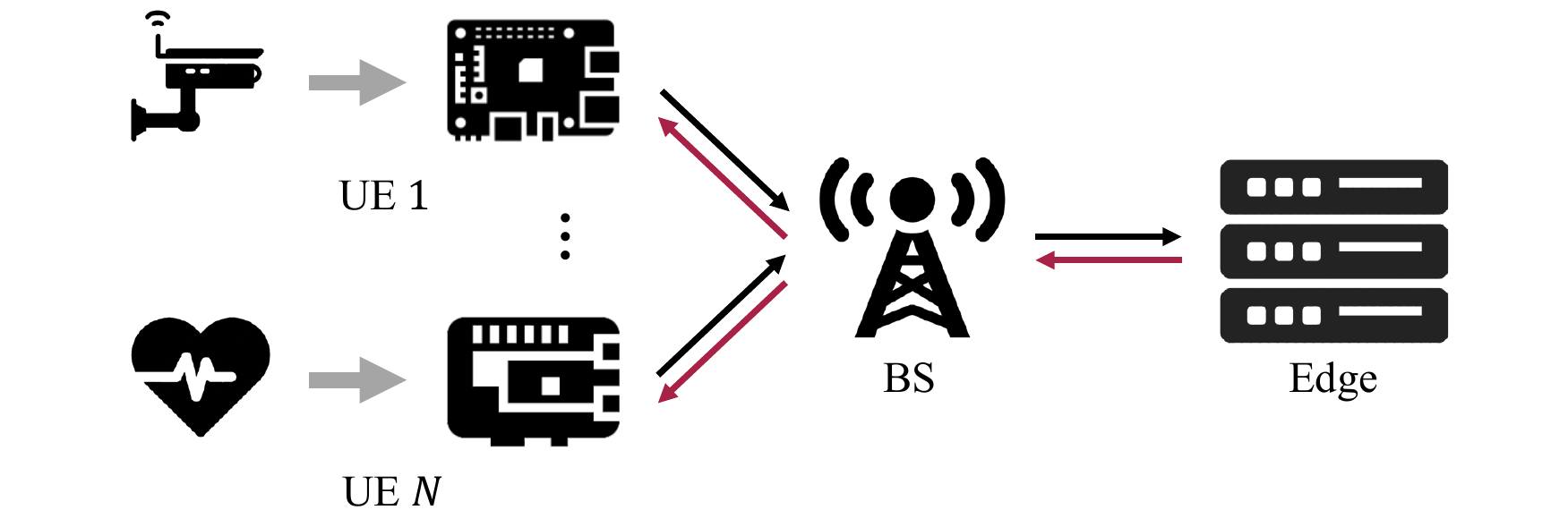}
  \label{fig:workflow_a}}
  \vfil
  \subfloat[]{\includegraphics[width=0.475\textwidth]{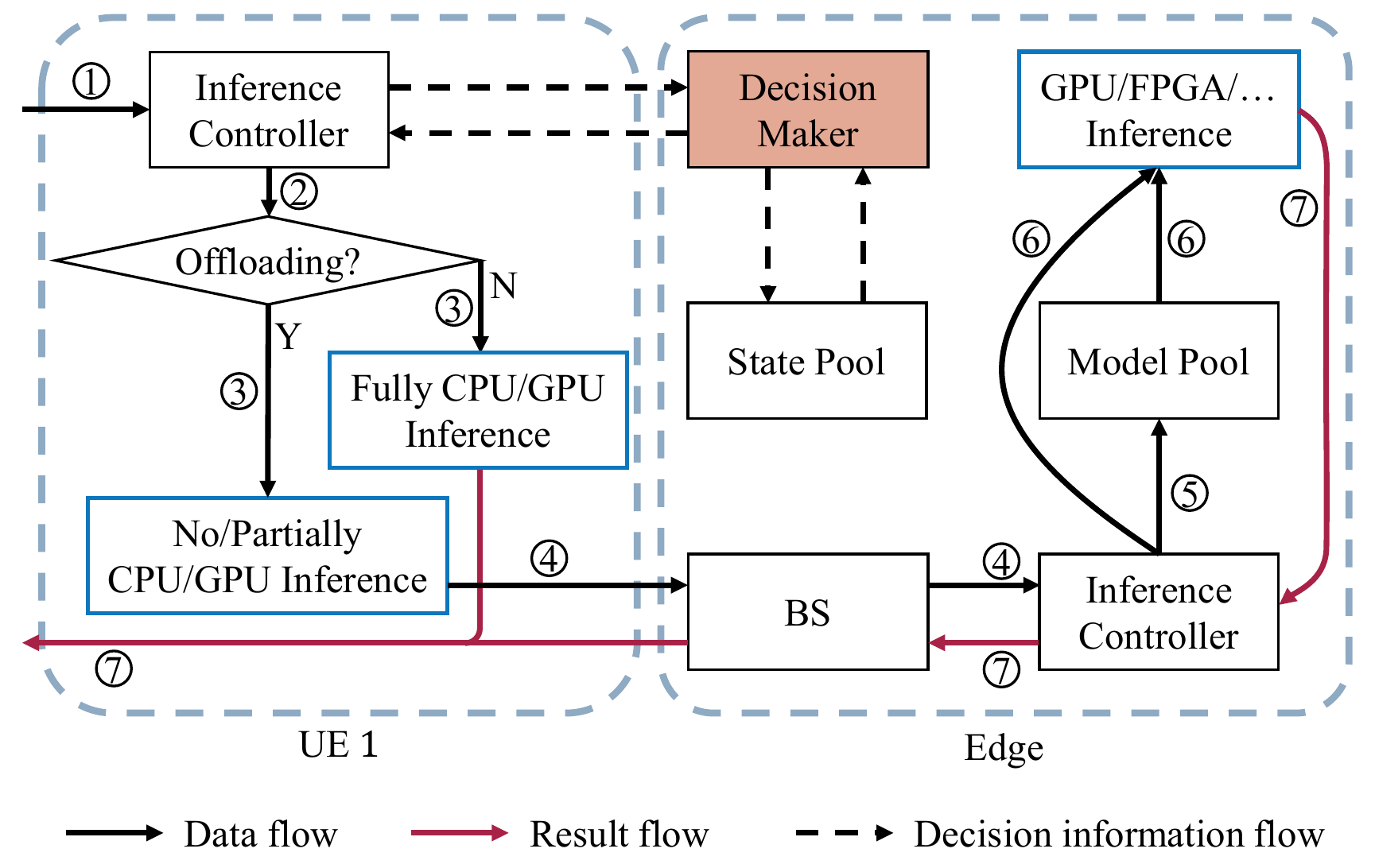}
  \label{fig:workflow_b}}
  \caption{Workflow of the multi-agent collaborative inference scenario: (a) the overall workflow; (b) workflow of a specific UE, where the UE can perform local inference, offload the original input to the edge for inference, or split the DNN into two parts for collaborative inference.}
  \label{fig:workflow}
\end{figure}

In practice, one edge server usually needs to provide service for more than one UE, where the optimal offloading policy is more difficult to obtain than the single UE scenario, because of the interference within wireless channels.
In this section, we consider a scenario where multiple UEs conduct DNN inference with the help of only one edge server.
We will first present the workflow of the multi-agent collaborative inference, and then the system model and problem formulation.

\subsection{Workflow}

The overall workflow of the multi-agent collaborative inference system is provided in Fig. \ref{fig:workflow}(a), where the system contains multiple UEs, a wireless base station (BS), and an edge server.
UEs communicate with the BS via a wireless channel, and the BS is linked with the edge server via an optical fiber network.
Each UE needs to complete several DNN inference tasks.
For each task, the edge server makes a decision about where the task is to be processed, \textit{i.e.}, on the UE holding the task or on the edge server.
Our goal is to find the decision that helps complete all tasks quickly while consuming as little energy as possible.

Fig. \ref{fig:workflow}(b) illustrates the detailed workflow by taking ``UE $1$'' as an example.
We consider a fixed-frequency decision update for the collaborative inference problem and divide time into nonoverlapping frames.
At the beginning of each time frame, a decision-maker deployed at the edge server sends decisions about where tasks should be processed during the current frame to each UE.
The decision is obtained based on the state of each UE, such as the number of remaining tasks, the size of data to be offloaded, and so on.
We provide a detailed definition of the UE state in Sec. \ref{sec:mdp_sub}.
At the end of the previous frame, each UE sends their state information to the edge server, and the edge server collects and stores the states of all UEs.
We term the collection of all UE states the state pool.
In a specific time frame, when the UE starts a new task, it follows the received decision on whether to offload the task to the edge or not.
If the decision is to perform inference locally, the UE will execute the complete neural network with its on-chip CPU/GPU.
Otherwise, the UE will partially (or not) run the neural network locally and then transmit the compressed intermediate features (or the original input) to the BS via a wireless channel.
The BS will then transmit such data to the edge server through an optical fiber.
At the edge end, the server will identify the right model according to the received data at first and then complete the inference task using its more powerful hardware.
Finally, the edge server returns the result to the UE.

\subsection{DNN Inference Model}

Collaborative inference requires to decouple DNN into several parts.
Here we assume that the DNN can be divided by layers or residual blocks, if it is a deep residual model.

For a system consists of $N$ UEs, we denote the set of UEs as $\mathcal N$, $\mathcal N=\{1,...,N\}$.
The model deployed on UE $n$ can be divided into $B_n+1$ parts, which indicates that there are $B_n$ potential partitioning points for the model deployed on UE $n$.
The partitioning decision of this UE can be denoted as $b_n\in\mathcal B$, $\mathcal B=\{0,1,...,B_n+1\}$, which indicates that the UE will offload the intermediate result to the edge after completing the inference of the first $b_n$ parts.
If $b_n=0$, UE $n$ will direct offload the original input to the edge server for inference, and if $b_n=B_n+1$, the UE will conduct the inference locally.

\subsection{Communication Model}
\label{sec:computation_model}

During the offloading procedure, UEs first communicate with a wireless BS.
Then the data will be transmitted to the edge server via an optical fiber network.

Each user transmits data through a particular offloading channel with a transmit power.
We denote the channel and the power of UE $n$ as $c_n\in \{1,...,C\}$ and $p_n>0$, respectively.
Combined with the partitioning point $b_n$, the three terms construct the inference action of UE $n$, denoted as $(b_n,c_n,p_n)$.
The decision-maker provides inference action of all UEs, based on a policy $\pi$, which is the probability distribution of all possible inference actions.
Hence, the uplink data rate between UE $n$ and the edge server can be computed as \cite{rappaport1996wireless}:
\begin{equation}
  r_{n}(\pi)=\omega_{c_n} \log_2\left(1+\frac{p_n g_n}{\sigma_{c_n} +\sum\limits_{i\in \mathcal N\backslash\{n\},b_i\in\mathcal B\backslash\{B_i+1\}} p_i g_i}\right),
\end{equation}
where $\omega_{c_n}$ is the bandwidth of channel $c_n$, $g_n$ is the channel gain between UE $n$ and the wireless BS, and $\sigma_{c_n}$ is the background noise power of channel $c_n$.

\subsection{Computation Model}

We consider the inference latency and the energy consumption of each sample.
The overhead for collecting states of the UEs in each time frame is ignored, as the state information is usually negligible compared with the offloading data.
In the worst case, the extra latency at the server is only a small constant, which does not affect the optimization results.
Here we still take UE $n$ as an example, and suppose that the size of compressed intermediate feature is $f_n$ bits at partitioning point $b_n$.
When $n=0$, $f_n$ denotes the size of the original input sample.
We assume a powerful edge server, \textit{i.e.}, all received tasks can be finished with negligible latency, and omit the latency at the edge end.
Similar settings can be found in previous works \cite{DBLP:journals/ejwcn/ChenW20}.
Hence, the inference latency of a single sample is composed of three terms: the local inference latency $t_n^f$, the feature compression latency $t_n^c$, and the data transmission latency $t_n^t$.
The local inference latency $t_n^f$ and the feature compression latency $t_n^c$ can be collected on device.
Based on the communication model, the data transmission latency is:
\begin{equation}
  t_{n}^t=\frac{f_n}{r_{n}(\pi)}.
\end{equation}

Thus the overall latency can be computed as:
\begin{equation}
  \begin{aligned}
    t_{n}(\pi)=t_n^f I_{\{b_n\in\mathcal B\backslash\{0\}\}}&+t_n^c I_{\{b_n\in\mathcal B\backslash\{0,B_n+1\}\}}\\&+t_{n}^t I_{\{b_n\in\mathcal B\backslash\{B_n+1\}\}},
  \end{aligned}
\end{equation}
where $I_{\{C\}}$ is an indicator function which equals to 1 only when the condition $C$ is true and equals to 0 otherwise.

Similarly, we can derive the energy consumption as:
\begin{equation}
  \begin{aligned}
    e_{n}(\pi)=e_n^f I_{\{b_n\in\mathcal B\backslash\{0\}\}}&+e_n^c I_{\{b_n\in\mathcal B\backslash\{0,B_n+1\}\}}\\&+e_{n}^t I_{\{b_n\in\mathcal B\backslash\{B_n+1\}\}},
  \end{aligned}
\end{equation}
where $e_n^f$ and $e_n^c$ are the local inference energy consumption and the feature compression energy consumption, respectively, which can also be measured on device; $e_{n}^t$ is the data transmission energy consumption, which is computed as:
\begin{equation}
  e_{n}^t = p_n t_{n}^t.
\end{equation}

\subsection{Problem Formulation}

Based on the above modelings, we formulate the overall optimization problem.
By solving this problem, we can achieve effective multi-agent collaborative inference.

In our scenario, each UE receives several tasks at the beginning, and our goal is to find a policy $\pi$ that can minimize both the latency and energy consumption for completing all tasks.
Supposing that UE $n$ receives $K_n$ tasks, then our optimization problem can be formulated as:
\begin{equation}
  \begin{array}{r@{\quad}l@{}l}
    \text{(P1)}&\min\limits_{\pi}&[\max\limits_n\sum_{k=1}^{K_n}t_{n,k}(\pi)+\beta\sum_{n=1}^N\sum_{k=1}^{K_n}e_{n,k}(\pi)  \vspace{1ex}]\\
    \text{s.t.}&\text{(C1)}&~b_n\in\{0,1,...,B_n+1\},~\forall n\in\mathcal N\\
    &\text{(C2)}&~c_n\in\{1,...,C\},~\forall n\in\mathcal N\\
    &\text{(C3)}&~0< p_n\leq p_{max},~\forall n\in\mathcal N,
  \end{array}
\end{equation}
where $p_{max}$ is the maximum transmit power of a single UE and $\beta>0$ is a balancing hyperparameter.

\section{MDP Reformulation}
\label{sec:mdp}

In this section, we reformulate the problem as an MDP to facilitate the optimization procedure.
Table \ref{tab:notation} is a summary of the used notation.

\subsection{Problem Redefinition}

Problem (P1) is a mixed-integer nonlinear programming problem \cite{bonami2012algorithms}, which is NP-hard and difficult to solve for an optimal solution via traditional approaches \cite{DBLP:journals/corr/abs-2001-09223}.
DRL has emerged as a promising method to solve such problems, which requires to represent the problem as an MDP.
However, problem (P1) cannot be directly expressed in the MDP form, since it contains the sum of latency and the energy consumption of each task.
When state transition occurs, there may be half-completed tasks. 
This means the latency and the energy consumption of tasks in the last time interval may not be unavailable.
Thus, we should reformulate the problem for applying DRL algorithms.

To achieve this goal, we propose to consider the averaged latency and energy consumption of each completed task in each time frame and ignore the overhead of half-completed tasks.
We enforce the averaged latency of each completed task in each time frame to be small so that more tasks can be completed, and thus the overall latency will also be small.
We also require the averaged energy consumption to be small, so that the overall energy consumption can be minimized.
Supposing that we require $T(\pi)$ time frames to complete all tasks, and the expected number of completed tasks in one time frame and the energy consumption at time frame $t$ are $K(\pi)$ and $E_t(\pi)$, respectively, then the problem can be reformulated as:
\begin{equation}
  \begin{array}{r@{\quad}l}
    \text{(P2)}&\min\limits_{\pi}~\sum_{t=1}^{T(\pi)}\frac{T_0 + \beta E_t(\pi)}{K(\pi)}\\
    \text{s.t.}&\text{(C1),~(C2),~and~(C3)},
  \end{array}
\end{equation}
where $T_0$ denotes the duration of one time frame.

\begin{table}[!t]
  \caption{Notation}
  \centering
  \begin{tabular}{l|l}
    \hline
    \textbf{Symbol} & \textbf{Description}                                 \\
    \hline
    $\pi$           & policy for generating action $a_t$                   \\
    $T(\pi)$        & time horizon                                         \\
    $T_0$           & duration of one time frame                           \\
    $K(\pi)$        & expected number of completed tasks in one time frame \\
    $K_t(\pi)$      & number of completed tasks at time $t$                \\
    $E_t(\pi)$      & energy consumption at time $t$                       \\
    $\mathbf{k}_t$  & numbers of uncompleted tasks at time $t$             \\
    $\mathbf{l}_t$  & left time of local computation at time $t$           \\
    $\mathbf{n}_t$  & size of left data to be offloading at time $t$       \\
    $\mathbf{d}$    & distance from UEs to the edge server                 \\
    $\mathbf{b}_t$  & partitioning point of each UE at time $t$               \\
    $\mathbf{c}_t$  & offloading channel of each UE at time $t$            \\
    $\mathbf{p}_t$  & transmit power of each UE at time $t$            \\
    $R$             & cumulative reward in one episode                     \\
    $\phi$          & parameters of the critic network                     \\
    $\theta_n$      & parameters of the $n$-th actor network               \\
    \hline
  \end{tabular}
  \label{tab:notation}
\end{table}

\subsection{Theoretical Analysis}

In this section, we prove that the objective of problem (P1) is upper bounded by that of problem (P2), so that we can enhance multi-agent collaborative inference by optimizing problem (P2).
This is summarized in the following theorem.

\begin{theorem}
  A policy update that minimizes the objective function of problem (P2) can also minimize the objective function of problem (P1) for small $\beta$.
  \label{theorem}
\end{theorem}

\begin{proof}
  We assume that the numbers of UEs and the task number of each UE in the two problems are equal.
  For convenience, we signify the objective function of problem (P1) and (p2) as $f_1(\pi)$ and $f_2(\pi)$, respectively. 
  We denote the first and the second terms in $f_1(\pi)$ as $A(\pi)$ and $B(\pi)$, i.e., $f_1(\pi)=A(\pi) + \beta B(\pi)$, and $f_2(\pi)$ can be rewritten as:
  \begin{equation}
    \notag
    \begin{aligned}
      f_2(\pi)&=\frac{\sum_{t=1}^{T(\pi)}\left(T_0 + \beta E_t(\pi)\right)}{K(\pi)}\\
      &=\frac{T(\pi)T_0 + \beta \sum_{t=1}^{T(\pi)}E_t(\pi)}{K(\pi)}\\
      &=\frac{A(\pi) + \beta B(\pi)}{K(\pi)}.
    \end{aligned}
  \end{equation}

  If there are two policies $\pi_1$ and $\pi_2$ that satisfy $f_2(\pi_1)>f_2(\pi_2)$, we aim to prove that $f_1(\pi_1)>f_1(\pi_2)$.
  We consider two cases: 1) $K(\pi_1)\geq K(\pi_2)$; 2) $K(\pi_1)< K(\pi_2)$.

  For case 1), it is obvious that $f_1(\pi_1)>f_1(\pi_2)$.

  For case 2), we can obtain that $T(\pi_1)>T(\pi_2)$, since $K(\pi)$ is the expected number of completed tasks in one time frame.
  Thus we have $A(\pi_1)>A(\pi_2)$, and for $f_1(\cdot)$, we have:
  \begin{equation}
    \notag
    \begin{aligned}
      f_1(\pi_1)-f_1(\pi_2)&=(A(\pi_1)-A(\pi_2))+\beta (B(\pi_1)-B(\pi_2)),
    \end{aligned}
  \end{equation}
  where $B(\pi_1)-B(\pi_2)$ denotes the extra energy used for completing all tasks using policy $\pi_2$.
  This value is upper bounded since the energy consumption of the local inference and the feature compression is constant, and the transmit power is upper bounded.
  Thus there must exist a small $\beta$ that ensures $f_1(\pi_1)>f_1(\pi_2)$.
  According to the analysis in the above two cases, we conclude that minimizing the objective function of problem (P2) can guarantee the decrease of the objective function of problem (P1) when $\beta$ is small.
\end{proof}

\textbf{Discussion.} Problem (P2) converts the sum of tasks into the sum of time frames, so that the problem is in accordance with the form of cumulative reward in DRL.
In practice, the improvement of $K(\pi)$ can usually decrease the overall energy consumption $B(\pi)$.
Thus generally a large $\beta$ can also guarantee the validity of Theorem \ref{theorem}.

\subsection{MDP Formulation}
\label{sec:mdp_sub}

\begin{figure*}
  \centering
  \includegraphics[width=\textwidth]{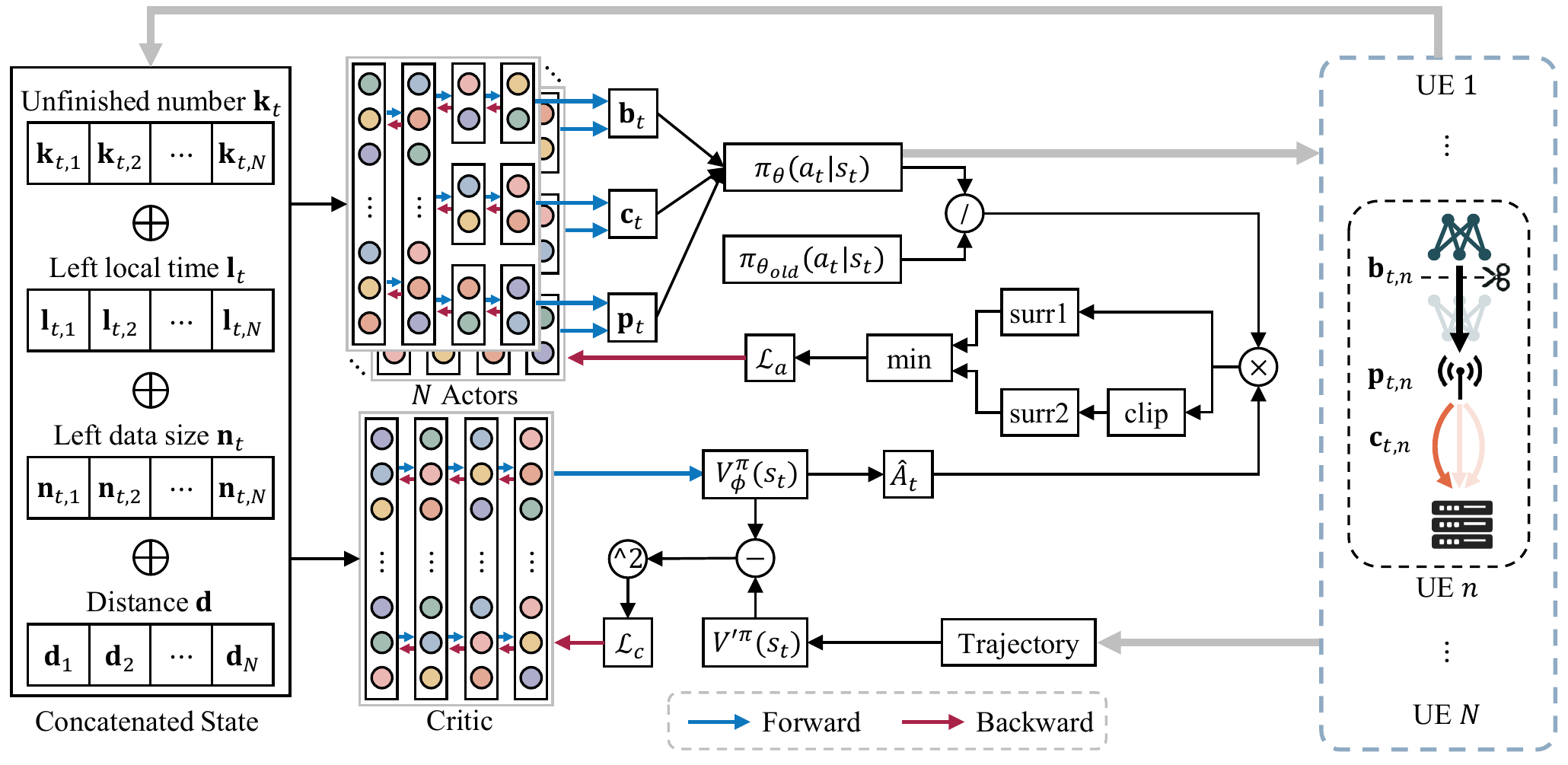}
  \caption{Overall structure of MAHPPO. 
  This algorithm contains multiple actor networks and a global critic network. 
  At the left part, all vectors in the state are concatenated and sent to the actors and the critic.
  Each actor outputs the partitioning point, offloading channel, and transmit power decisions for its corresponding UE, and the critic outputs an estimation of the state value.
  The middle part depicts the loss function, where ``surr'' is short for ``surrogate objective.''
  The right part illustrates how a trajectory for policy updating is obtained.}
  \label{fig:mahppo}
\end{figure*}

An MDP depicts the interaction between the DRL agent and the environment and can be represented by a tuple $(\mathcal S;\mathcal A;P;r)$, where $\mathcal S$ is a set of states, $\mathcal A$ is a set of actions, $P:\mathcal{S} \times \mathcal{A} \times \mathcal{S} \to \mathbf{R}$ is a probability distribution that depicts the system dynamics, and $r:\mathcal{S} \times \mathcal{A} \times \mathcal{S} \to \mathbf{R}$ is the reward.

We first define state $s_t$ of the MDP, where $t$ denotes the $t$-th decision time frame.
State $s_t$ consists of four parts: the left task number $\mathbf{k}_t$, the left local computation time $\mathbf{l}_t$, the left offloading data size $\mathbf{n}_t$, and the distance from UEs to the edge server $\mathbf{d}$.
That is, $s_t=\{\mathbf{k}_t,\mathbf{l}_t,\mathbf{n}_t,\mathbf{d}\}$, where $\mathbf{k}_t$ contains the left uncompleted task number of each UE at time $t$, which may guide the system to assign more resources to the UEs that have more uncompleted tasks;
$\mathbf{l}_t$ is the left time of completing the local inference and the feature compression of current half-completed tasks;
$\mathbf{n}_t$ is the size of left data to be offloaded of these tasks;
and the distance $\mathbf{d}$ is , which remains unchanged in one episode.

We then define action $a_t$ at time $t$, given by $a_t=\{\mathbf{b}_t,\mathbf{c}_t,\mathbf{p}_t\}$, where $\mathbf{b}_t$ is the partitioning point of each UE, $\mathbf{c}_t$ is the offloading channel of each UE, and $\mathbf{p}_t$ is the transmit power of each UE.
At time $t$, the transmit power $\mathbf{p}_t$ becomes effective immediately, while the other two elements in $a_t$ become effective when new tasks are started, i.e., after half-completed tasks in the previous time frame are completed.

Finally, we define reward $r_t$.
The expectation $K(\pi)$ in problem (P2) is hard to obtain.
Thus we use the number of completed tasks at the $t$-th time frame $K_t(\pi)$ as an estimate of $K(\pi)$. 
According to the problem definition in problem (P2), we define the reward as:
\begin{equation}
  r_t=-\frac{T_0}{K_t}-\beta\frac{E_t}{K_t},
\end{equation}
where we have omitted the dependence on $\pi$ to simplify the notation.
The cumulative reward is $R=\sum_{t=1}^{T(\pi)}r_t$, which is just the negative of the objective function of problem (P2).
Thus we can solve the problem by maximizing the cumulative reward.

\section{MAHPPO Algorithm}
\label{sec:mahppo}

In this section, we propose a MAHPPO algorithm to solve the multi-agent collaborative inference problem.
We first depict the network design of the algorithm and then present the optimization procedure.

\subsection{Actor-Critic Architecture Design}
\label{sec:ac_net}

The proposed MAHPPO method has an actor-critic structure, and Fig. \ref{fig:mahppo} is an illustration of the method.
The multi-agent collaborative inference scenario contains multiple UEs, so we propose to adopt multiple actor networks to provide inference decisions for the UEs.
The number of actor networks is equal to the number of UEs.
A major challenge is how to deal with the hybrid action space, and we handle this by adding three output branches in each actor networks to acquire hybrid actions \cite{DBLP:conf/ijcai/FanS0019}.
The three branches share several front layers that encode the state information and output partitioning point, offloading channel, and transmit power decisions, respectively. 
We denote parameters of the critic network and the actor networks as $\phi$ and $\theta_n$, respectively, where $\theta_n$ denotes the parameters of the corresponding actor network of UE $n$.

The input of the actor-critic network is state $s_t$, as described in Sec. \ref{sec:mdp_sub}.
The state consists of four vectors that depict the number of uncompleted tasks of each user, left local computing time and offloading data size of half-completed tasks, and the distance to the edge server of each UE.
In practice, we simply concatenate the four vectors into one and input them into the actor-critic network.

The critic network outputs the predicted state value $V^{\pi}_{\phi}(s_t)$, which is the expected cumulative reward of $s_t$.
The state value guides the update of all actor networks in the training stage.

The actor networks output the policy $\pi_{\theta}(a_t|s_t)$, which is a predicted distribution of action $a_t$ at state $s_t$.
The branches of actors that provide discrete actions output the probabilities $\{p_i(s_t)\}$ of selecting different possible actions simultaneously, and this is achieved by adding a softmax function at the end of these branches.
Thus the discrete actions is followed by a categorical distribution:
\begin{equation}
  \begin{aligned}
    \pi_{\theta_n}^d(a_{t,n}^d|s_t)=&\prod_{i=1}^{M}p_i(s_t)I_{\{a_{t,n}^d=i\}},\\
    &\forall n\in \{1,2,...,N\},\sum_{i=1}^{M}p_i(s_t)=1,
  \end{aligned}
\end{equation}
where the superscript $d$ denotes the discrete part of the action and $M$ is the number of possible actions.

Meanwhile, the branches of actors that provide continuous actions output the mean $\mu(s_t)$ and the standard deviation $\sigma(s_t)$ values of the actions.
We assume the distribution of the continuous action is Gaussian, i.e.,
\begin{equation}
    \pi_{\theta_n}^c(a_{t,n}^c|s_t)\sim\mathcal{N}(\mu(s_t),\sigma^2(s_t)),
\end{equation}
where the superscript $c$ denotes the continuous part of the action and $\mathcal{N}(\cdot)$ is the probability density function of the Gaussian distribution.
In practice, the action can be sampled from the above distributions.

\subsection{Optimization Objectives}
\label{sec:optimizing_objectives}

We introduce the optimization objectives of the critic and the actors here.
The critic network should fit an unknown state-value function and the actor networks should provide policy to maximize the fitted state value.
The optimizing objectives should guide the critic and the actors to achieve these goals. 

We first present the objective of the critic network.
Suppose there exists a trajectory of the MDP, and the trajectory describes the interaction process between the DRL agent and the environment.
Then we can obtain the rewards of each time frame in the trajectory, and the real cumulative reward at state $s_t$ is:
\begin{equation}
  V'^{\pi}(s_t)=\sum_{t'=t}^{T(\pi)}\gamma^{t'-t}r_{t'},
  \label{eq:state_value}
\end{equation}
where $\gamma \in [0,1]$ is the discount factor that balances the long-term return and the short-term return.
We use this sampled value as the expected cumulative reward to train the critic network.
The loss function is given by:
\begin{equation}
  \mathcal{L}_c(\phi)=||V^{\pi}_{\phi}(s_t)-V'^{\pi}(s_t)||_2.
  \label{eq:loss_c}
\end{equation}

\begin{figure}[!t]
  \removelatexerror
    \begin{algorithm}[H]
    \caption{MAHPPO}
    Randomly initialize parameters of the critic and the actors as $\phi$ and $\theta_n,~n\in\{1,2,...,N\}$;\\
    Initialize learning rate $\alpha$, batch size $B$, sample reuse time $K$, and initial state $s_0$;\\
    Initialize trajectory buffer $\mathbf{M}$ with size $||\mathbf{M}||$;\\
    Current state $s_t\leftarrow s_0$;\\
    \For(){step $S:=1$ to $S_{max}$}{
      \tcp{Collecting trajectory}
      \While(){$\mathbf{M}$ is not filled}{
        Sample action $a_t\sim \pi_\theta(a_t|s_t)$;\\
        Execute $a_t$ and observe reward $r_t$, the next state $s_{t+1}$;\\
        Append $(s_t,a_t,r_t,s_{t+1})$ into $\mathbf{M}$;\\
        \uIf(){$s_{t+1}$ is the terminate state}{
          $s_t\leftarrow $ reinitialized state $s_0$\;
        }
        \Else(){
          $s_t\leftarrow s_{t+1}$\;
        }
        $S\leftarrow S+1$;
      } 
      \tcp{Updating Network}
      Compute state value for states in $\mathbf{M}$ with (\ref{eq:state_value});\\
      Compute advantage for states in $\mathbf{M}$ with (\ref{eq:gae});\\
      \For(){epoch $e:=1$ to $\lfloor K\times (||\mathbf{M}||/B)\rfloor$}{
        Sample $B$ samples from $\mathbf{M}$;\\
        Compute $\mathcal{L}_c(\phi)$ and $\mathcal{L}_a(\theta)$ with these samples;\\
        Update critic: $\phi \leftarrow \phi - \alpha \nabla_{\phi} \mathcal{L}_c(\phi)$;\\
        \For(){$n:=1$ to $N$}{
          Update actor: $\theta_n \leftarrow \theta_n - \alpha \nabla_{\theta_n} \mathcal{L}_a(\theta)$;\\
        }
      }
      Clear memories in $\mathbf{M}$.
    }
    \label{algo}
    \end{algorithm}
\end{figure}

We then present the objective of the actor networks.
According to the monotonic improvement guaranteed policy gradient algorithm, trust region policy optimization \cite{DBLP:conf/icml/SchulmanLAJM15}, we maximizes the following objective:
\begin{equation}
  \max_\theta ~\mathbf{E}_t \left[\frac{\pi_\theta (a_t|s_t)}{\pi_{\theta_{old}}(a_t|s_t)}\hat A_t\right],
\end{equation}
where $\pi_\theta (a_t|s_t)$ is the current policy and $\pi_{\theta_{old}}(a_t|s_t)$ is the old policy for collecting a trajectory.
Furthermore, $\hat A_t$ is the advantage function which measures how much a specific action $a_t$ is better than the average actions at state $s_t$.
This objective is called the ``surrogate objective'' because it uses the importance sampling trick to treat samples from the old policy as the surrogate of new samples to train the actors.
In order to reduce bias of the advantage function, we employ an exponentially-weighting method to obtain the generalized advantage estimation \cite{DBLP:journals/corr/SchulmanMLJA15}:
\begin{equation}
  \hat A_t = \sum_{t'=t}^{T(\pi)}(\gamma \lambda)^{t'-t}\left(r_t+\gamma V^{\pi}_{\phi}(s_{t+1})-V'^{\pi}(s_t)\right),
  \label{eq:gae}
\end{equation}
where $\lambda \in [0,1]$ is a hyperparameter.
If $t+1>T(\pi)$, we have $V^{\pi}_{\phi}(s_{t+1})=0$.

The loss clipping strategy has been demonstrated to be helpful to train the actor \cite{DBLP:journals/corr/SchulmanWDRK17}.
Denoting $\frac{\pi_\theta (a_t|s_t)}{\pi_{\theta_{old}}(a_t|s_t)}$ as $r_t(\theta)$, the clipped loss can be formulated as:
\begin{equation}
  \mathcal{L}^{CLIP}(\theta)=\mathbf{E}_t\left[\min(r_t(\theta)\hat A_t,\text{clip}(r_t(\theta),1-\epsilon,1+\epsilon)\hat A_t)\right],
\end{equation}
where $\epsilon>0$ is a hyperparameter that controls how $r_t(\theta)$ can move away from $1$.

Hence, we can formulate the objective function of the actor networks as:
\begin{equation}
  \mathcal{L}_a(\theta)=\sum_{n=1}^N \left[\mathcal{L}^{CLIP}(\theta_n)+\zeta \mathbf{E}_t[\mathcal{H}(\pi_{\theta_n}(a_t|s_t))]\right],
\end{equation}
where $\mathcal{H}(\pi_{\theta_n}(a_t|s_t))$ is an entropy bonus that encourages exploration and $\zeta$ is a balancing hyperparameter.
We summarize the proposed MAHPPO procedure in Algorithm \ref{algo}.
Each expectation term is evaluated by the averaged results of a batch of samples.

\section{Experiments}
\label{sec:exp}

We first present the results of the designed autoencoder-based intermediate feature compression method and measure the overhead of both local inference and feature compression.
We then show the convergence performance of the proposed MAHPPO algorithm and investigate how effectively it can reduce the inference overhead of multi-agent collaborative inference with the ResNet18 model.
Finally, we evaluate our framework on two other popular DNN models: VGG11 and MobileNetV2.
Code is available at \url{https://github.com/Hao840/MAHPPO}

\subsection{Intermediate Feature Compression Performance}

In this set of experiments, we select JALAD \cite{DBLP:conf/icpads/LiHJWWZ18} as the baseline.
JALAD compresses the intermediate feature by 8-bit quantization and entropy coding. 
According to their paper, 8-bit quantization causes almost no accuracy loss.

\begin{figure}[!t]
  \centering
  \includegraphics[width=\linewidth]{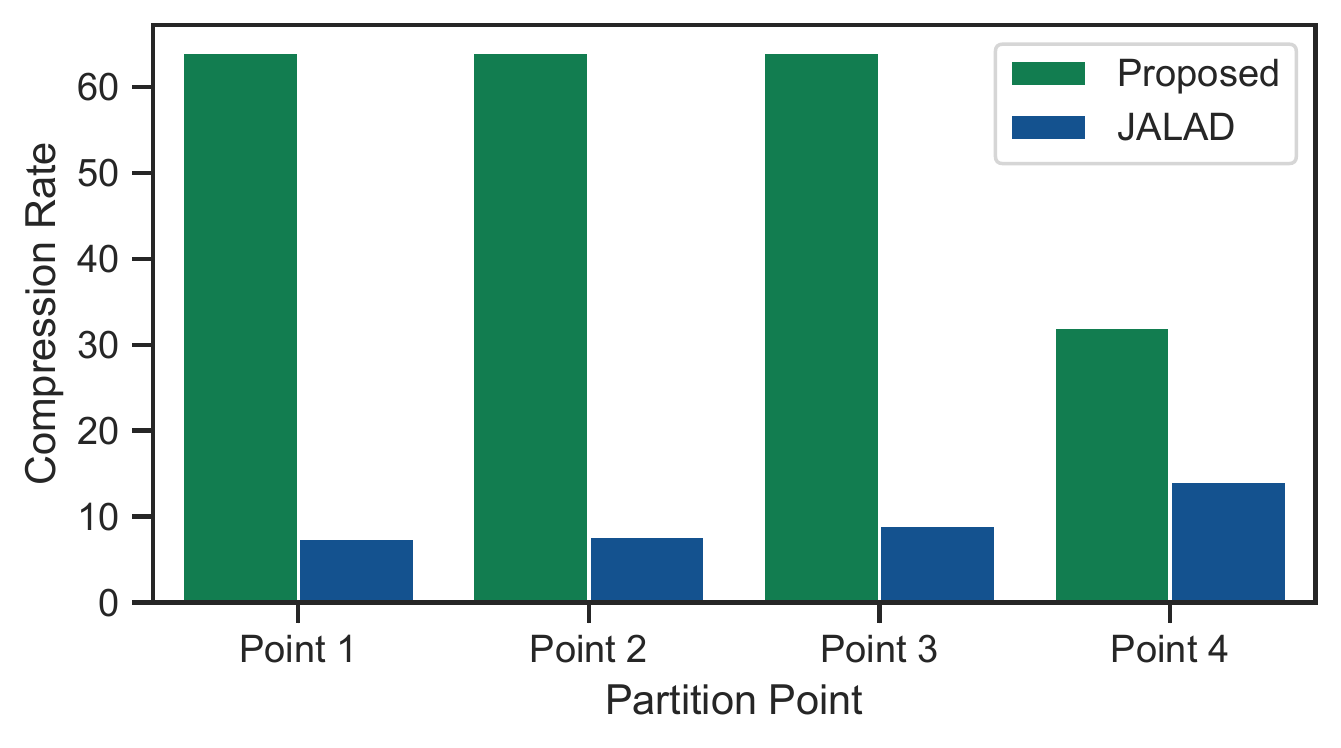}
  \caption{Compression rate comparison of the proposed lightweight autoencoder-based intermediate feature compression module and JALAD on ResNet18.}
  \label{fig:ae_main}
\end{figure}

\begin{figure}[!t]
  \centering
  \includegraphics[width=\linewidth]{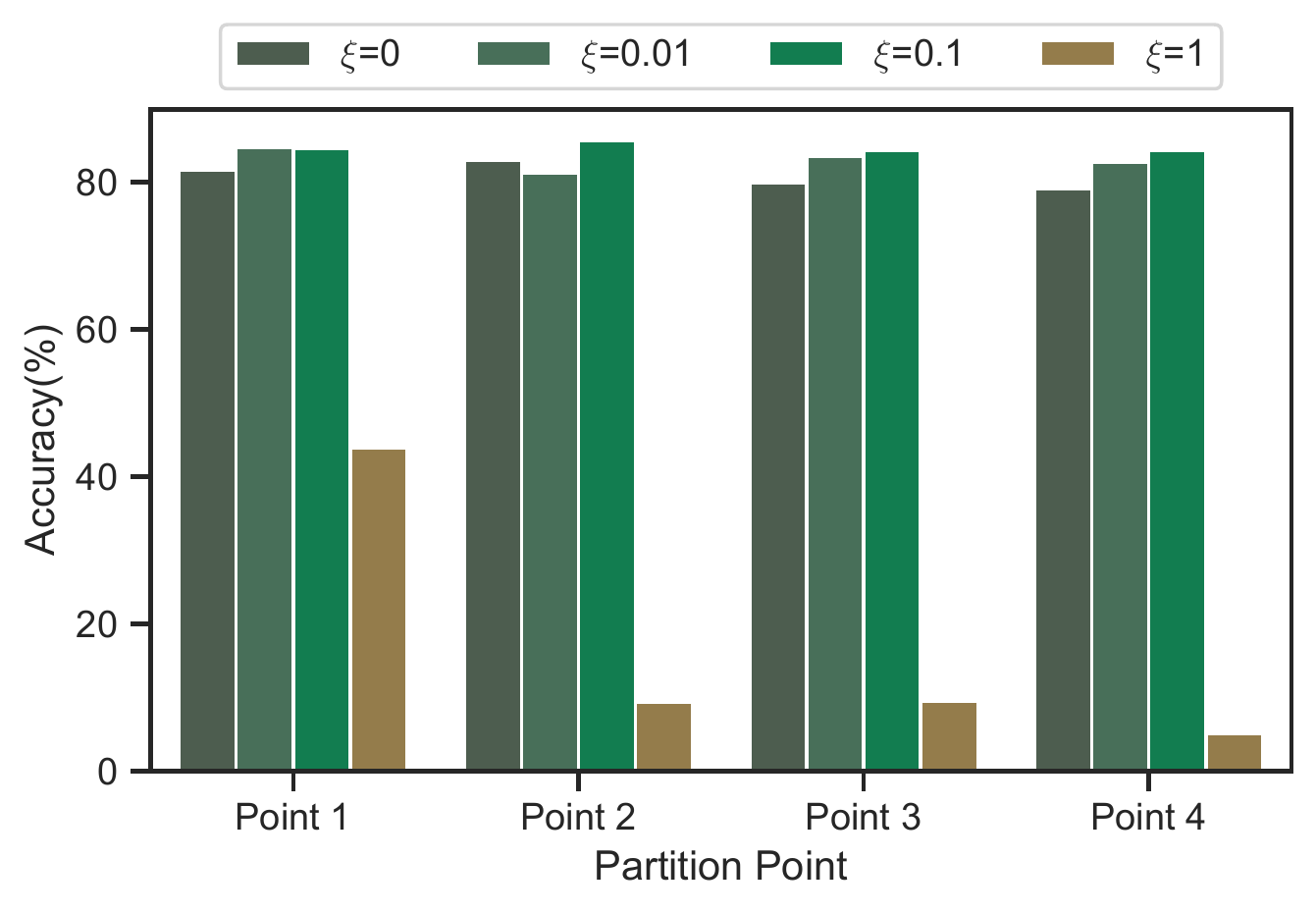}
  \caption{Comparison of different $\xi$ settings on ResNet18.}
  \label{fig:ae_xi}
\end{figure}

The experiments are conducted on the Caltech-101 dataset \cite{caltech101}, which is a computer vision dataset containing 101 classes of samples each with size 300$\times$200.
The number of samples in each class ranges from 80 to 300.
We randomly select 80\% of the samples in each class as training set and use the remaining samples as test set.
During training, each sample is first resized to 256$\times$256 and then randomly cropped to 224$\times$224 for data augmentation.
We train a ResNet18 network as the base model and select 4 partitioning points.
Specifically, a ResNet18 model processes a sample in four stages.
We adopt the output ends of the second layer in each stage, \emph{i.e.}, the batch normalization layer, as a partitioning point.
Hence, there are 4 selected points in total.
We train several autoencoders with different compression rates at each partitioning point and use 8-bit quantization to further compress the encoded feature.
Although a more significant compression rate always achieves less latency and energy consumption, we cannot increase the rate unboundedly, since the compression rate also affects the accuracy.
Hence, we select the autoencoder to achieve the maximum compression rate under a 2\% accuracy loss bound after fine-tuning, since 2\% accuracy loss is negligible for most real-world deployments, compared with the reduced latency and energy consumption.
In cases where performance degradation is unacceptable, we can replace the autoencoder with lossless compression modules, \emph{e.g.}, the entropy coding.
We train each autoencoder for 30 epochs with a 128 batch size.
The optimizer is Adam \cite{DBLP:journals/corr/KingmaB14} and the learning rate is 0.1.
The hyperparameter $\xi$ is set to 0.1.
For fine-tuning, we train the base model and the autoencoder together for 10 epochs with a small learning rate of 0.0001.

Fig. \ref{fig:ae_main} illustrates the evaluation result of the lightweight autoencoder-based intermediate feature compression module.
Compared with JALAD, it can be seen that our method can compress the intermediate feature much more effectively.
When the partitioning point moves towards the tail of the pre-trained model, the compression rate of our method decreases and that of JALAD increases.
This may because the features in the pre-trained model become sparser in deeper layers, and the simple autoencoder is not adequate to compress such data, while entropy coding is more effective for sparse data compressing.
In practice, the higher compression rate performance of our method will energize the multi-agent collaborative inference remarkably.

We further conduct experiments to compare different $\xi$ settings at different partitioning points.
Fig. \ref{fig:ae_xi} illustrate the comparison result.
The result shows that the $\xi=0.1$ setting performs the best, except at the partitioning point 1, where the performance is only slightly worse than the $\xi=0.01$ setting.
For convenience, we set $\xi$ to 0.1 in all the cases.

\subsection{Local Overhead Measurement}

\begin{figure}[!t]
  \centering
  \includegraphics[width=0.975\linewidth]{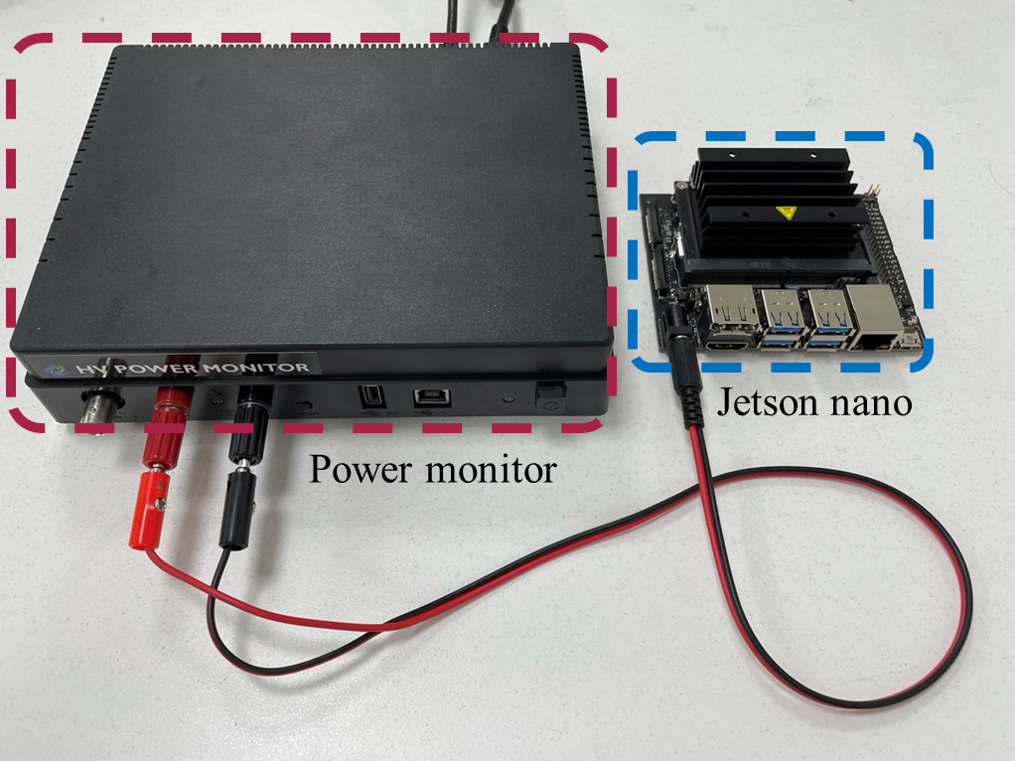}
  \caption{Inference overhead measurement hardware.}
  \label{fig:devices}
\end{figure}

We then measure latency and energy consumption of the local inference and feature compression, as described in Sec. \ref{sec:computation_model}.
We use an NVIDIA Jetson Nano as the UE and an external high voltage power monitor for the measurement.

The NVIDIA Jetson Nano contains a Quad-core ARM Cortex-A57 MPCore processor at 1.6 GHz, a 4 GB RAM Memory, and an NVIDIA Maxwell architecture GPU with 128 NVIDIA CUDA cores.
Jetson nano is driven by the power monitor, which can collect the output voltage and current at a sampling rate of 5000 samples per second.
The two devices are connected as in Fig. \ref{fig:devices}.
Before the measurement, we switch the Jetson Nano to the low power mode, i.e., the max power is 5 Watt, and turn off the dynamic voltage and frequency scaling (DVFS). 
We perform the inference on the test set of Caltech-101 dataset using Jetson nano.
The model is the trained ResNet18 with feature compression autoencoders.
We measure the inference overhead with 1000 samples in the test set.
We ignore the data of the first 100 samples and take the average of the latency and power of the last 900 samples as the data of the steady system.
We minus the power of the standby system from the averaged power to acquire the power of performing inference.
Thus the energy consumption of a single inference can be obtained by taking the product of the averaged latency and the power of performing inference.

The measurement results are shown in Fig. \ref{fig:overhead}.
The top figure illustrates the latency of local inference and feature compression at each partitioning point, and the bottom one illustrates the energy consumption.
The gray dashed line denotes the overhead when executing the full model locally.
We can see that the proposed method brings nearly no additional latency and energy consumption.
Moreover, the overhead of our method is less than that of full local inference at each partitioning point, except for the energy consumption at the last partitioning point.
The results also show that the energy cost of only running the model before the partitioning point 4 is larger than that of running the whole model.
We conjecture that this part of the model has a higher degree of parallelism for processing, as it is only consists of convolution layers.
This results in a higher average power and a less latency, while the energy, the product of the two values, becomes larger.
If we carefully assign each UE a partitioning point, an offloading channel, and a transmit power in the multi-agent collaborative inference scenario, the overall overhead can still be less than that of full local inference.
In contrast, as shown in the results, JALAD incurs more overhead than full local inference in most cases, which is because the large intermediate features need plenty of time to perform entropy coding.

\subsection{MAHPPO Convergence Performance}

After obtaining the local inference and feature compression overhead, we solve problem (P2) with the proposed MAHPPO algorithm.
We first present the experiment setups and then provide the convergence performance of MAHPPO.
Finally, we compare different parameter settings.

\subsubsection{Setup}
\label{sec:setup}

\begin{figure}[!t]
  \centering
  \includegraphics[width=\linewidth]{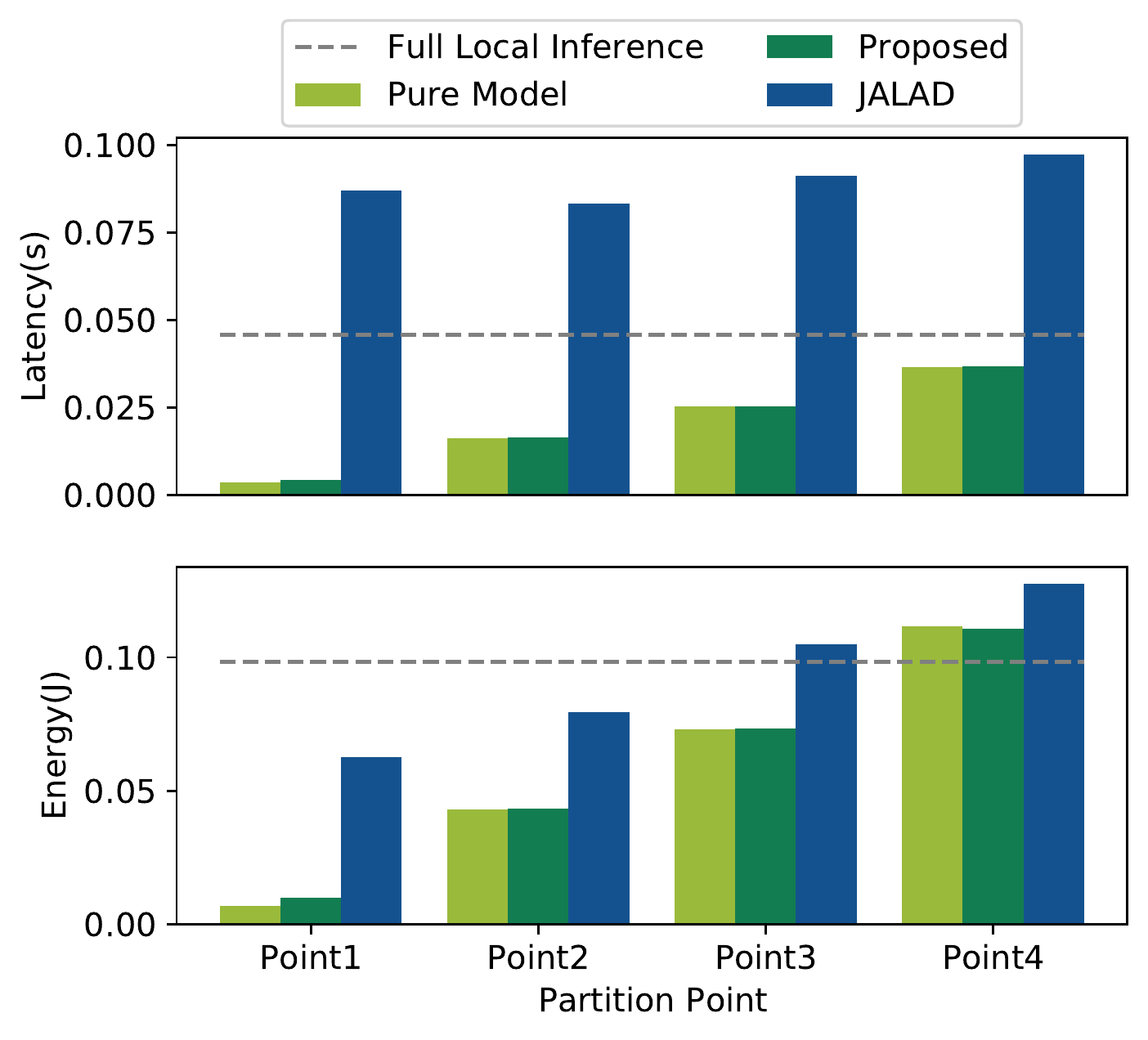}
  \caption{Latency and energy consumption of executing the front part of ResNet18 and compressing the intermediate feature on UE. The overhead of executing the full model on the UE is marked out by the gray dashed line.}
  \label{fig:overhead}
\end{figure}

\textbf{Environment.} 
In our experiment, there are $N=5$ UEs and $C=2$ channels.
The distance between UE $n$ and the edge server follows the uniform distribution $d_n\sim U[1,100]$ in meters.
UE $n$ will receive $K_n\sim Pois(\lambda_p)$ tasks at the beginning, where $Pois(\lambda_p)$ denotes the Poisson distribution, and $\lambda_p$ is its parameter set to 200.
When all tasks are completed, the current episode terminates and a new episode starts.
Distance $d_n$ and the number of tasks $K_n$ are re-initialized at the beginning of each episode.
The multi-agent collaborative inference system is assumed to locate in the urban cellular radio environment \cite{rappaport1996wireless}, where the channel gain is $g_n=d_n^{-l}$, where $l$ is the path loss exponent set to $3$.
We adopt a static channel setting, where the bandwidth of each channel is set to $\omega_{c_n}=1$MHz and the background noise power is set to $\sigma_{c_n}=10^{-9}$W \cite{DBLP:journals/ejwcn/ChenW20}.
The discount factor $\gamma$ is set to 0.95.
Furthermore, from the expression of the reward function, the expectation of completed tasks in one time frame is estimated by completed tasks in each time frame.
To ensure the accuracy of the estimation, the duration of each time frame $T_0$ should not be set to a too-small value.
Conversely, a large time frame setup helps to achieve accurate estimations, but the long offloading policy update interval also harms the performance due to the lack of flexibility.
In order to balance precision and flexibility, we set $T_0$ to 0.5 second, which is about 10 times larger than the latency of executing a full model inference on UE.
We roughly set $\beta$ to the ratio of local inference latency to energy consumption, i.e., 0.47.
In Sec. \ref{sec:beta}, we will study the impact of different $\beta$ settings.
When evaluating the performance of a trained agent, we always set $d_n=50$ and $K_n=200$ for a fair comparison.

\begin{figure}[!t]
  \centering
  \includegraphics[width=\linewidth]{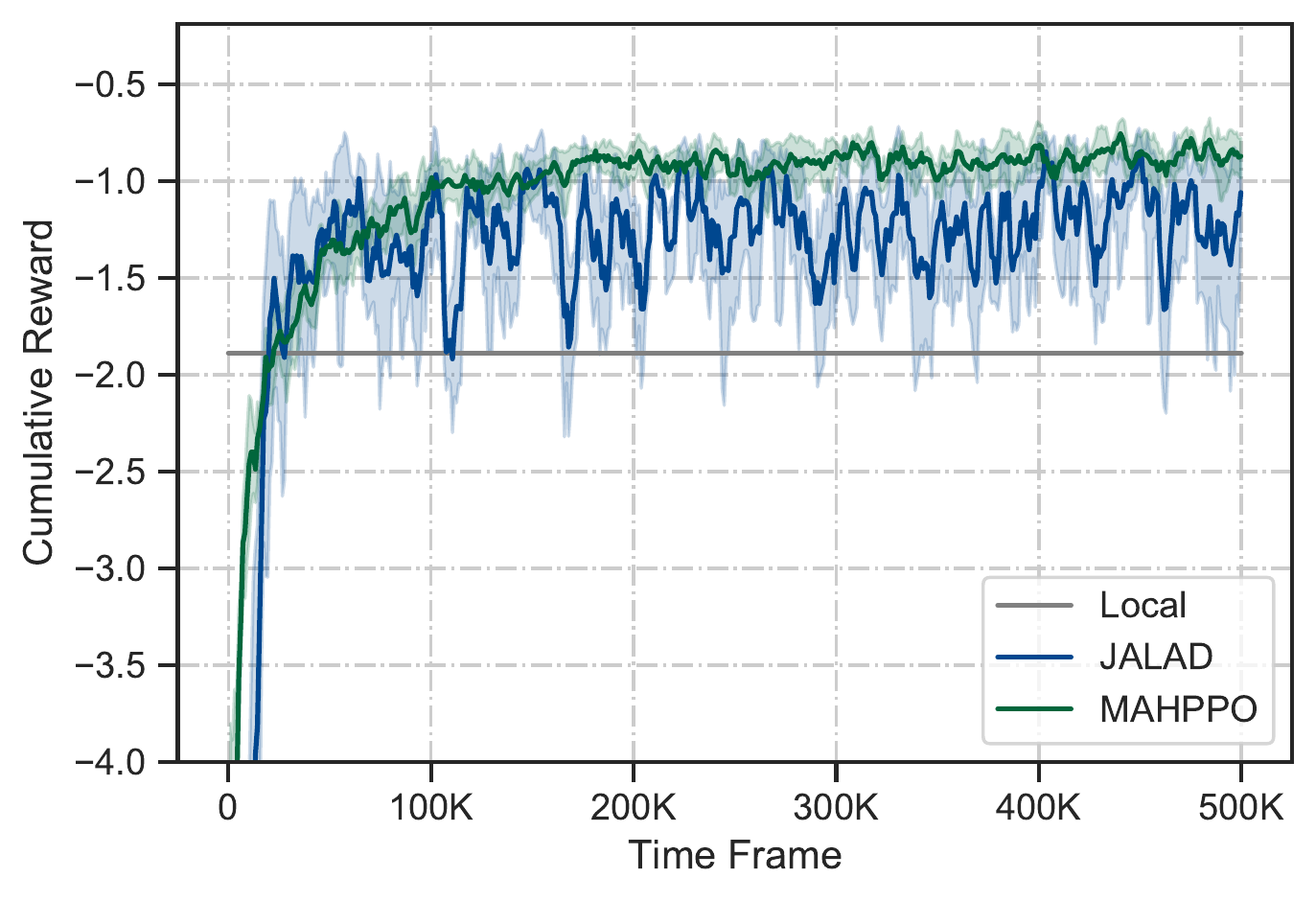}
  \caption{Convergence performance of the proposed MAHPPO algorithm and the two baselines with ResNet18.}
  \label{fig:convergence}
\end{figure}

\textbf{Agent.}
Each actor network is composed of fully connected layers, where the first 2 layers are the shared layers containing 256 and 128 neurons, respectively.
Each output branch also has 2 layers, where the first layer contains 64 neurons and the structure of the last layer is determined by the type and dimension of its corresponding action, as described in Sec. \ref{sec:ac_net}.
The critic network is composed of 4 fully connected layers with 256, 128, 64, and 1 neurons, respectively.
We train the actors and the critic for 50K steps with the same learning rate 0.0001.
The size of memory buffer $||\mathbf{M}||$ is 1024, the batch size $B$ is 256, and the sample reuse time $K$ is 10.
The setting of hyperparameters $\lambda$, $\epsilon$, and $\zeta$ follows the common setting in PPO implementation, which are 0.95, 0.2, and 0.001, respectively.

\textbf{Baselines.}
We use the following baselines:
\begin{itemize}
  \item \textbf{Local:} UEs always execute all tasks locally without the help of the edge.
  \item \textbf{JALAD:} JALAD is adopted as the feature compression method. The partitioning point, the offloading channel, and the transmit power are still obtained by the MAHPPO algorithm. The time frame is relaxed to 3 second to help convergence.
\end{itemize}

\subsubsection{Convergence Performance}
\label{sec:convergence_performance}

\begin{figure}[!t]
  \centering
  \subfloat[]{\includegraphics[height=0.33\linewidth]{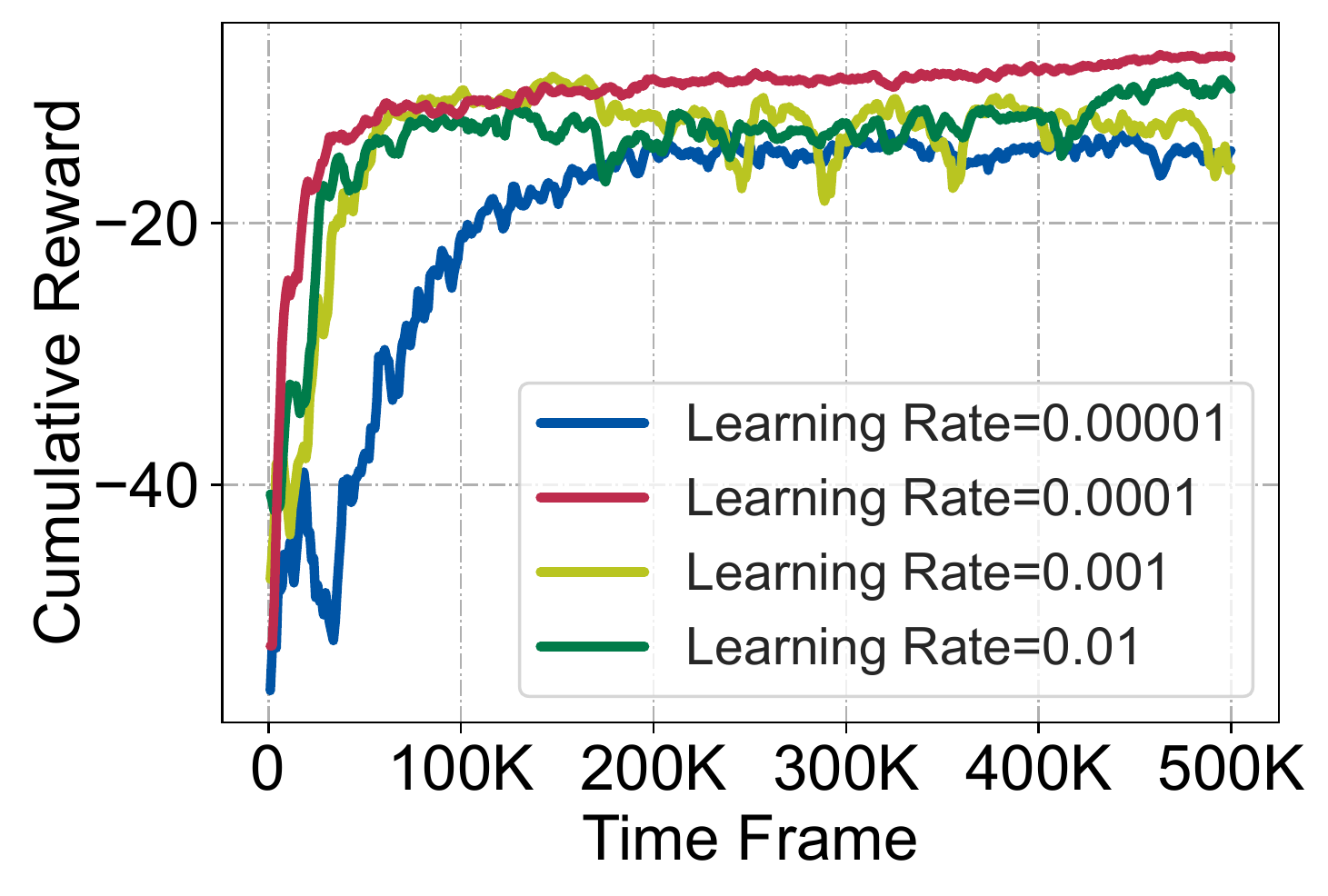}
    \label{fig:compare_lr_reward}}
  \subfloat[]{\includegraphics[height=0.33\linewidth]{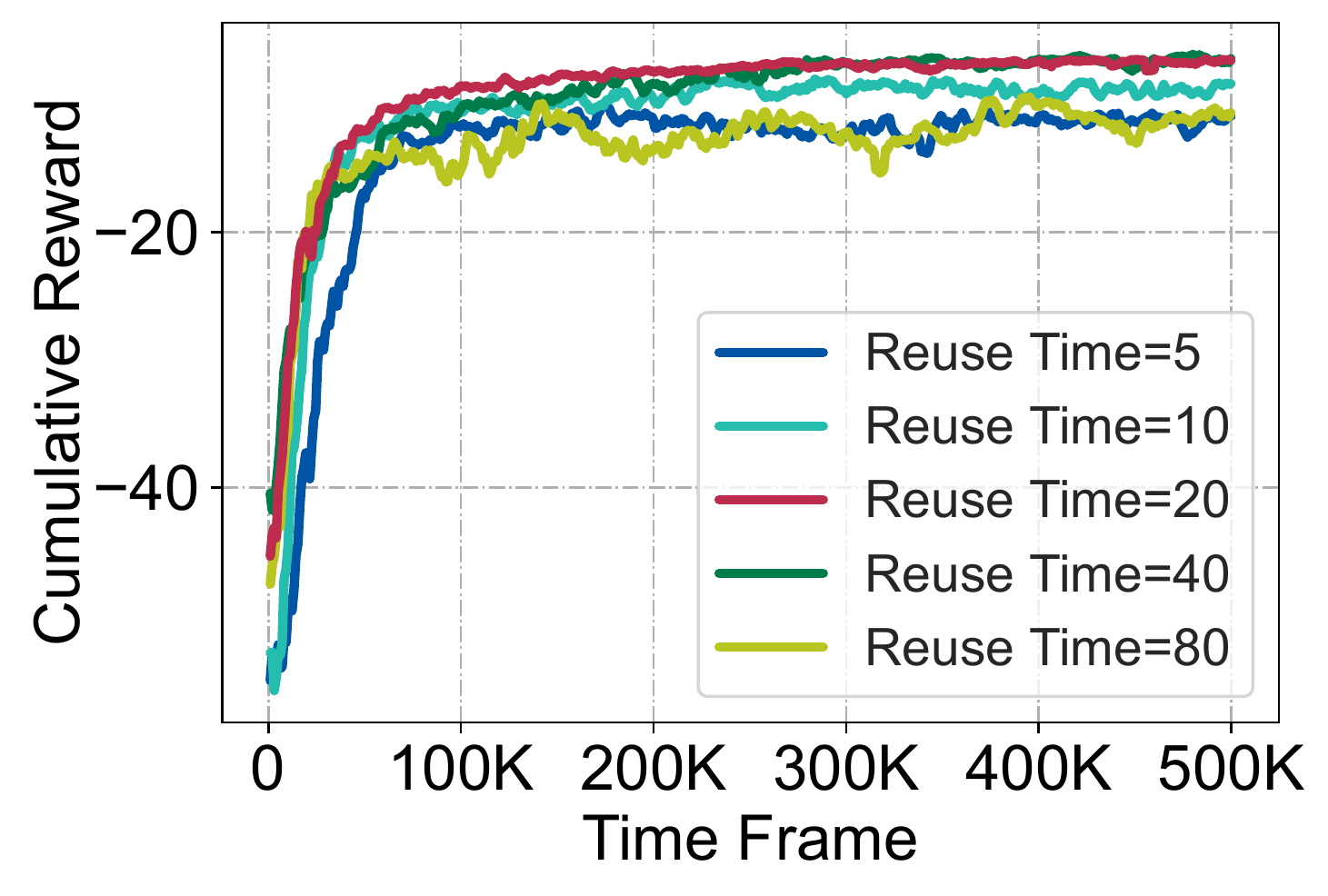}
    \label{fig:compare_repeat_reward}}
  \vfil
  \subfloat[]{\includegraphics[height=0.33\linewidth]{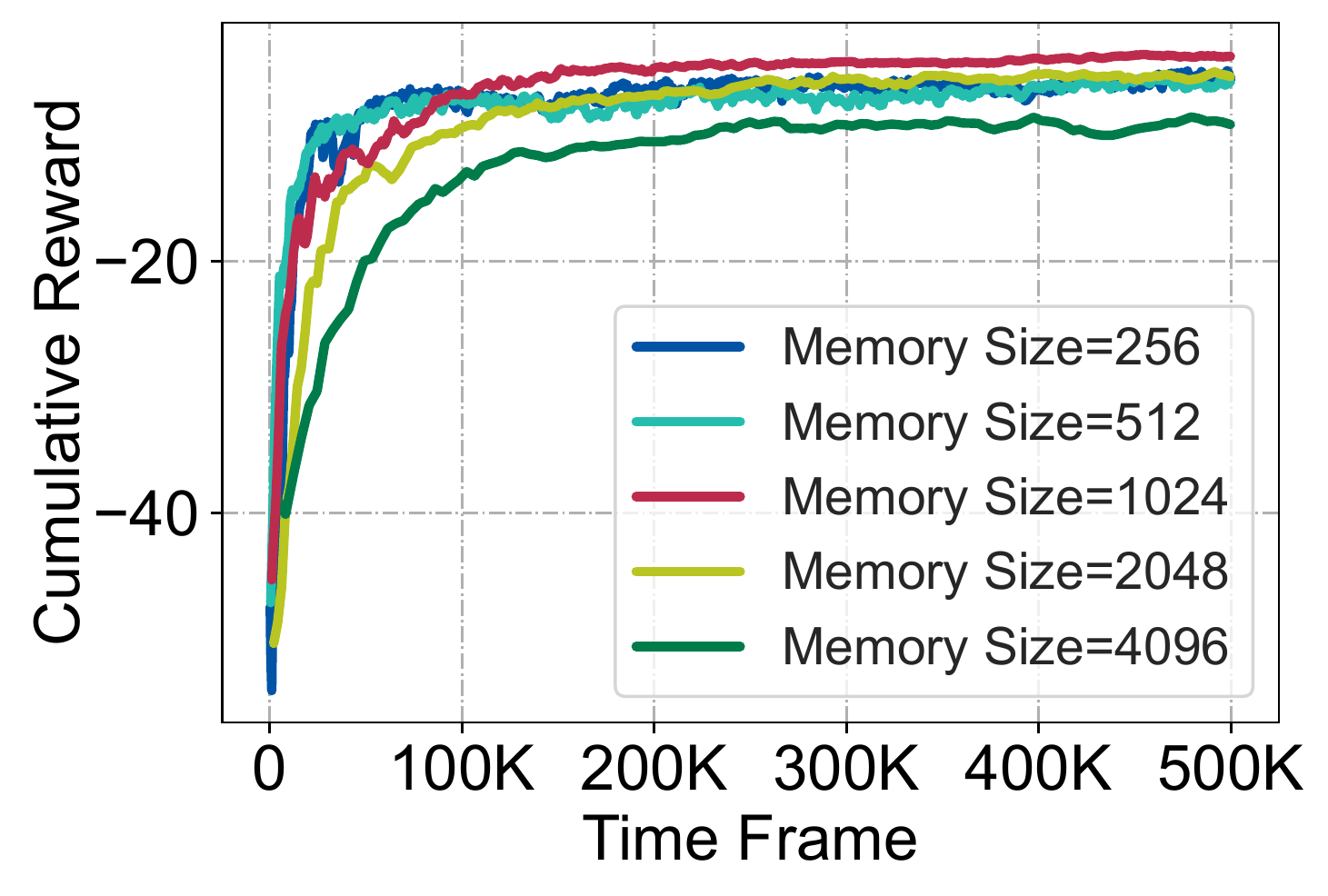}
    \label{fig:compare_step_reward}}
  \subfloat[]{\includegraphics[height=0.33\linewidth]{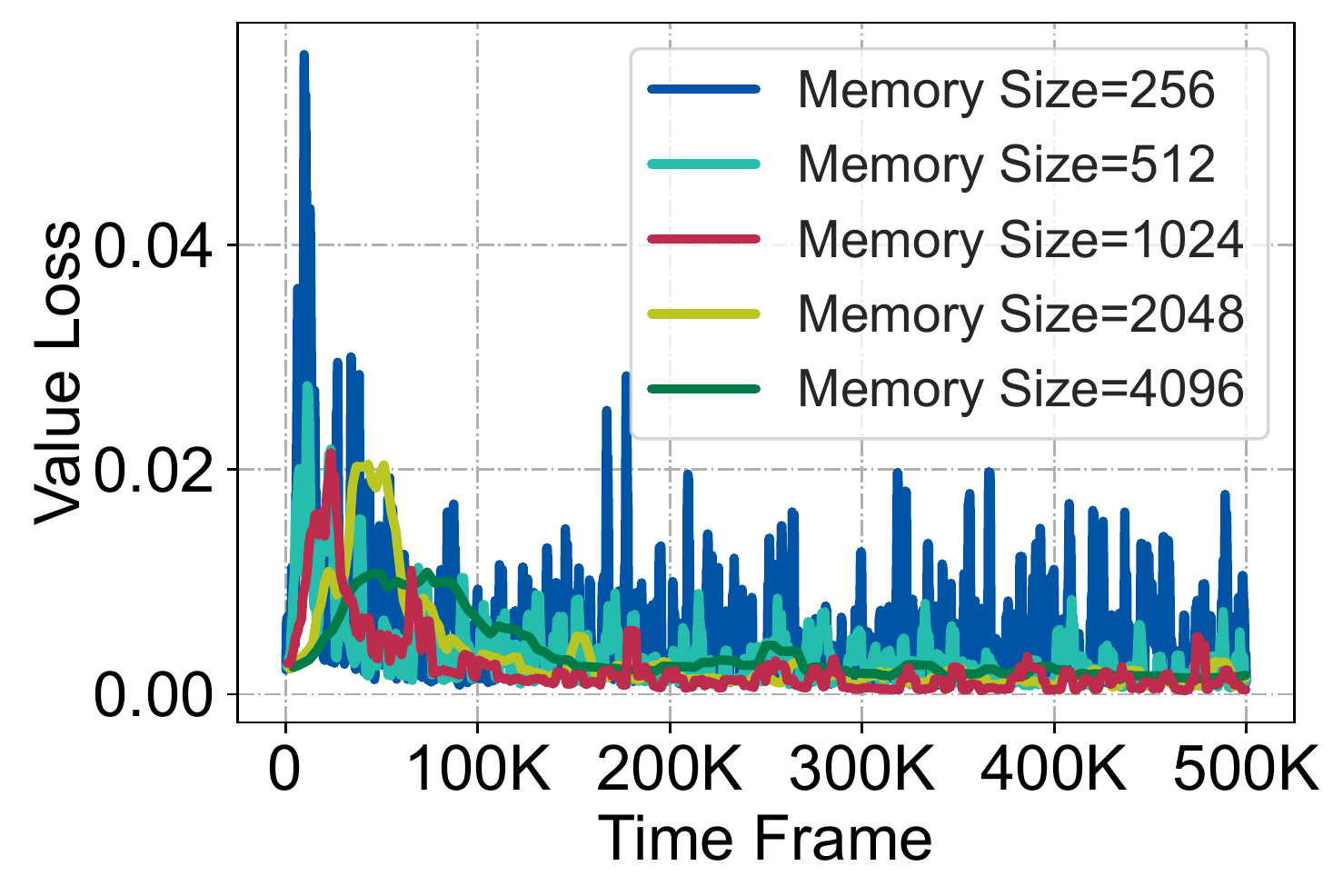}
    \label{fig:compare_step_loss}}
  \caption{Comparison of different parameter settings: (a) learning rate; (b) sample reuse time; (c, d) memory size.}
  \label{fig:compare}
\end{figure}

We present the training results of the proposed MAHPPO algorithm and the two baselines.
Each experiment setup is performed 5 times.

Fig. \ref{fig:convergence} illustrates the convergence performance of the MAHPPO algorithm and the two baselines.
The curves are smoothed by taking the average of the 5 nearest values at each point.
From the result, we can see that each method converges, and the MAHPPO algorithm performs the best.
The convergence curve of JALAD has a large fluctuating range.
The full local inference method executes all tasks on the UEs, and does not take advantage of the powerful edge computation resources, so it also performs worse than MAHPPO.
Note that the time frame of JALAD is 6 times larger than that of MAHPPO, so our total number of used time frames is much less.
Meanwhile, the reward value at each time frame is only determined by the expected overhead of a single inference, so we can roughly regard that the cumulative reward of JALAD is shrunk by 6 times.
Given this consideration, we find that JALAD performs the worst, as the complex intermediate feature compressor of JALAD makes it inefficient to decouple DNN at intermediate layers. Also, the large compressed features cause severe interference in the wireless channel.
In this case, executing all tasks locally is the best policy. Still, it will be hard for the DRL agent to learn, as the significant interference makes the environment extremely complex.
By taking the 6 times shrinkage into consideration, JALAD even performs worse than local inference.

\subsubsection{Hyperparameter Analysis}

We study how different main hyperparameter settings influence the convergence performance.
Different learning rates, sample reuse time, and memory size settings are compared.
The results are illustrated in Fig. \ref{fig:compare}.

In Fig. \ref{fig:compare}(a), the results indicate that training the agent with a small learning rate converges slow, while a large learning rate incurs unstable cumulative reward and hinders the agent to explore the best policy .
So we choose 0.0001 as the learning rate of the actors and the critic.
The sample reuse time indicates how many times a sample is used to train the agent.
In Fig. \ref{fig:compare}(b), we can see that a small sample reuse time setting incurs a slow and poor convergence.
Given increased reuse time, the convergence becomes faster and the policy becomes better.
However, a too large reuse time setup, e.g., 80, also results in a poor convergence.
So the sample reuse time is set to 20 to balance the accuracy and the computation complexity.
Fig. \ref{fig:compare}(c) and Fig. \ref{fig:compare}(d) illustrate the cumulative reward and the value loss of different memory size settings, respectively.
The value loss is defined in Eq. (\ref{eq:loss_c}).
In common PPO implementations, the batch size is set to a quarter of the memory size.
We follow this setting in the comparison.
From the result, we observe that training with a small memory size usually converges faster, since it leads to a high network update frequency.
However, the fast convergence leads to a poor policy since it over-uses the initial experiences.
Furthermore, a small memory size, which means that the batch size is also small, lead to unstable value loss decrease.
When the scenario becomes complicated, this may harm the convergence of the actor networks.
On the contrary, a large memory size setup also leads to poor performance, because it result in a long network update interval.
Hence, the memory size is set to 1024.
In the following experiments, we use the same hyperparameter settings.

\subsubsection{UE Number Comparison}

\begin{figure}[!t]
  \centering
  \includegraphics[width=\linewidth]{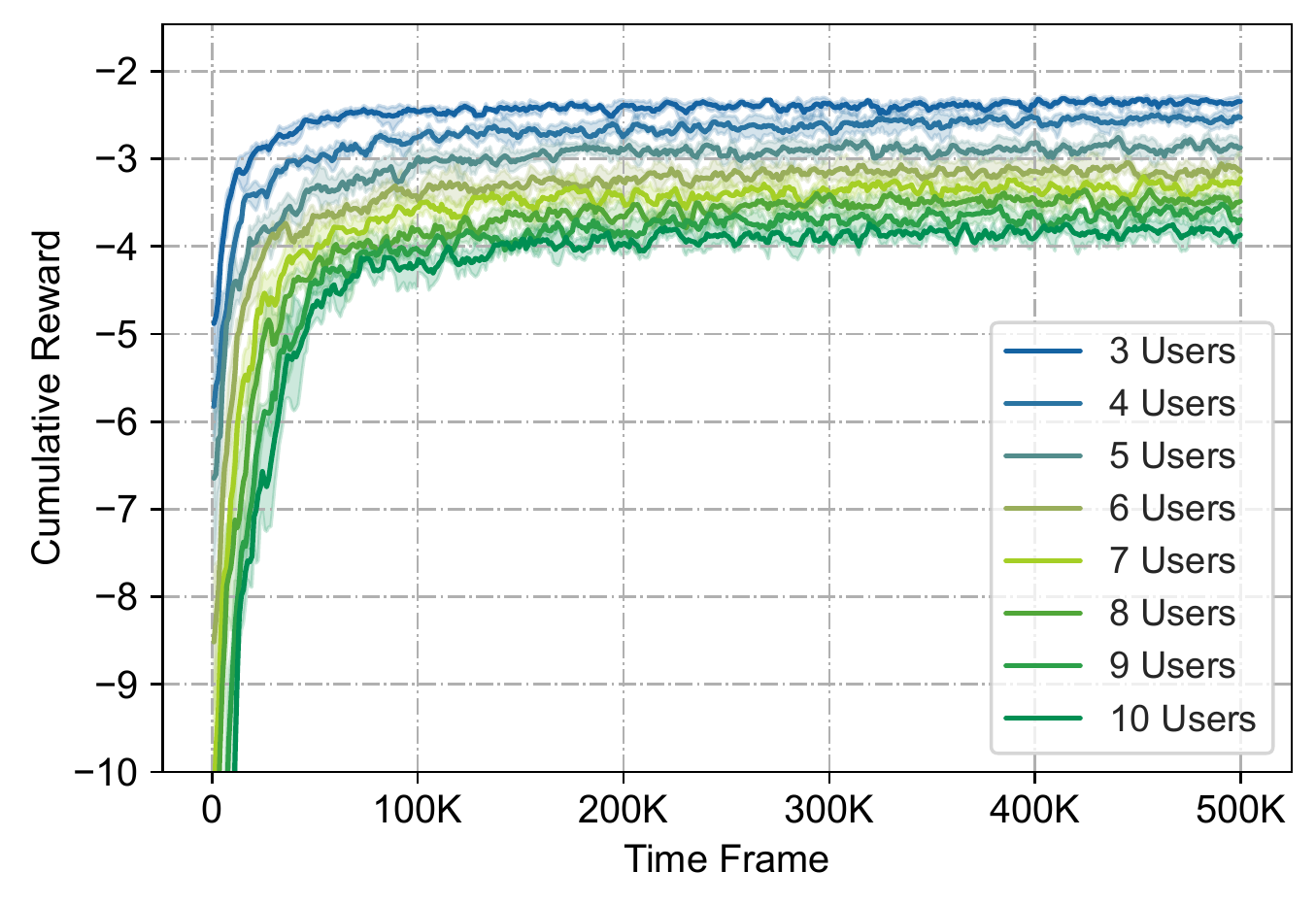}
  \caption{Convergence performance comparison of different UE number settings with ResNet18.}
  \label{fig:user_num_converge}
\end{figure}

We further compare the convergence performance of different numbers of UEs to study how the UE number influence the convergence when the number of channels is fixed.
The UE number $N$ ranges from 3 to 10.
Other details are the same as previous experiments.

Fig. \ref{fig:user_num_converge} illustrates the comparison result.
We observe that experiments using any UE number settings converge, and a larger UE number setting leads to a slower convergence.
This is because more UEs bring a more severe interference within wireless channels, and thus it is hard to find an optimal policy.
Furthermore, the convergent value of larger UE number setting tends to be smaller, since the overall wireless channel resources are fixed.

\subsection{Inference Latency and Energy Consumption}

We evaluate the performance of the MAHPPO algorithm by comparing the overhead of multi-agent collaborative inference.
We also compare the impact of different hyperparameter $\beta$ settings.
Evaluation is performed under the settings provided in Sec. \ref{sec:setup}.

\subsubsection{Overhead Saving Performance}

\begin{figure}[!t]
  \centering
  \includegraphics[width=\linewidth]{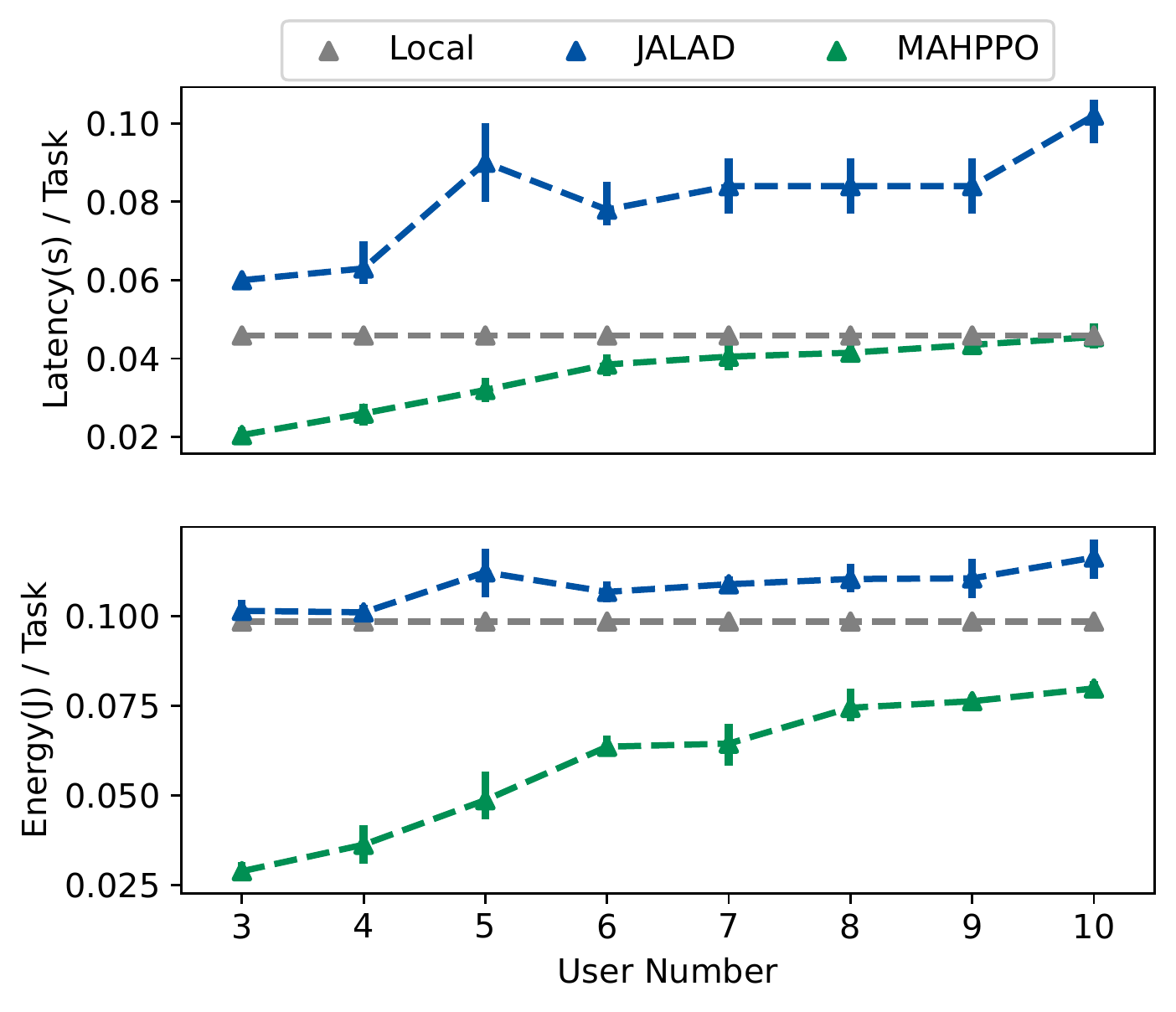}
  \caption{Averaged inference overhead of ResNet18 under different UE number settings.}
  \label{fig:user_num_overhead}
\end{figure}

We first evaluate the averaged inference latency and energy consumption of multi-agent collaborative inference when adopting the MAHPPO algorithm and the two baselines.
The results are provided in Fig. \ref{fig:user_num_overhead}.

We can see from the results that with increased number of UEs, the averaged inference latency and energy consumption of adopting the MAHPPO algorithm continue increase, as the available channel resources are fixed.
Also, the averaged inference latency and energy consumption of JALAD exhibit an increasing trend because the larger interference impedes the convergence more.
The averaged inference overhead of adopting the full local inference strategy remains constant because UEs do not need to compete for limited channel resources.
Compared with the full local inference strategy, the proposed method always achieves a lower inference overhead.
However, JALAD always performs worse than inference locally, which is in line with the analysis in Sec. \ref{sec:convergence_performance}.
When the UE number is $3$, our method can reduce $56\%$ of the inference latency and save $72\%$ of the energy consumption.
We can tune the hyperparameter $\beta$ to achieve a balance between latency and energy consumption.
Since the wireless channel resources are limited, the overhead curves of MAHPPO become close to that of local inference when the user number is increased.

\subsubsection{Impact of Hyperparameter \texorpdfstring{$\beta$}{beta}}
\label{sec:beta}

We study how the hyperparameter $\beta$ affects the inference overhead.
We set $\beta=0.01$, $0.1$, $1$, $10$, $100$, and $1000$.
The number of UE is set to 5.

Fig. \ref{fig:beta} presents the comparison results.
We run experiments for each $\beta$ setup 5 times, and the belts of shadows denote the standard deviation.
We can see that when the value of $\beta$ is increased, the inference latency increases and the energy consumption decreases.
When $\beta<0.1$, the change of its value has little impact on the inference overhead.
This is because the inference latency has a lower bound, the agent cannot further decrease the already small inference latency with little sacrifice of energy consumption.
When $\beta>0.1$, the inference latency and the energy consumption change significantly with $\beta$.
By adjusting the value of $\beta$, we can satisfy different overhead constraint requirements.

\begin{figure}[!t]
  \centering
  \includegraphics[width=\linewidth]{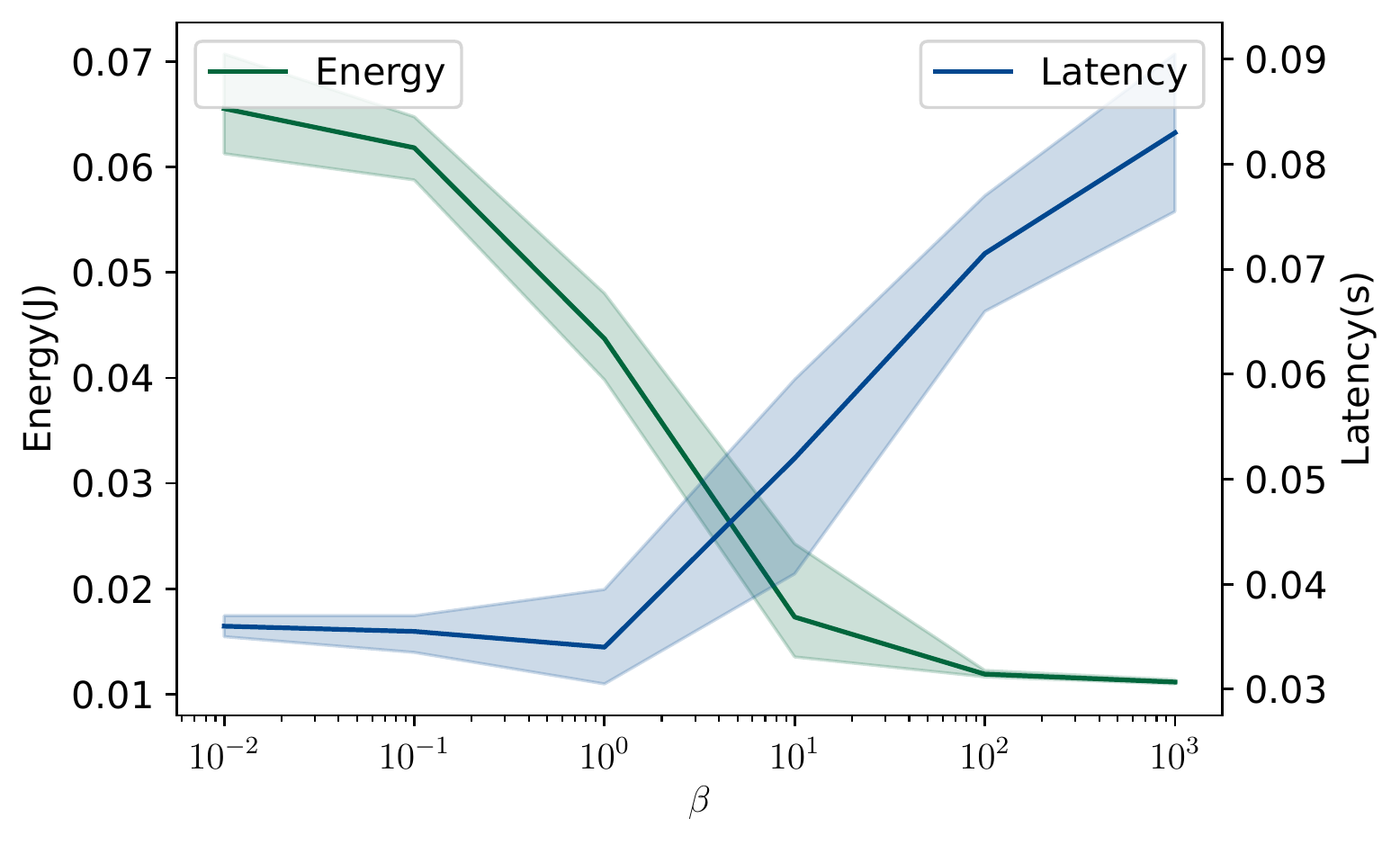}
  \caption{Impact of different hyperparameter $\beta$ settings. The belts of shadows present the standard deviation}
  \label{fig:beta}
\end{figure}

\subsection{Result with More Network Architectures}

\begin{figure*}[!t]
  \centering
  \subfloat[]{\includegraphics[width=0.475\linewidth]{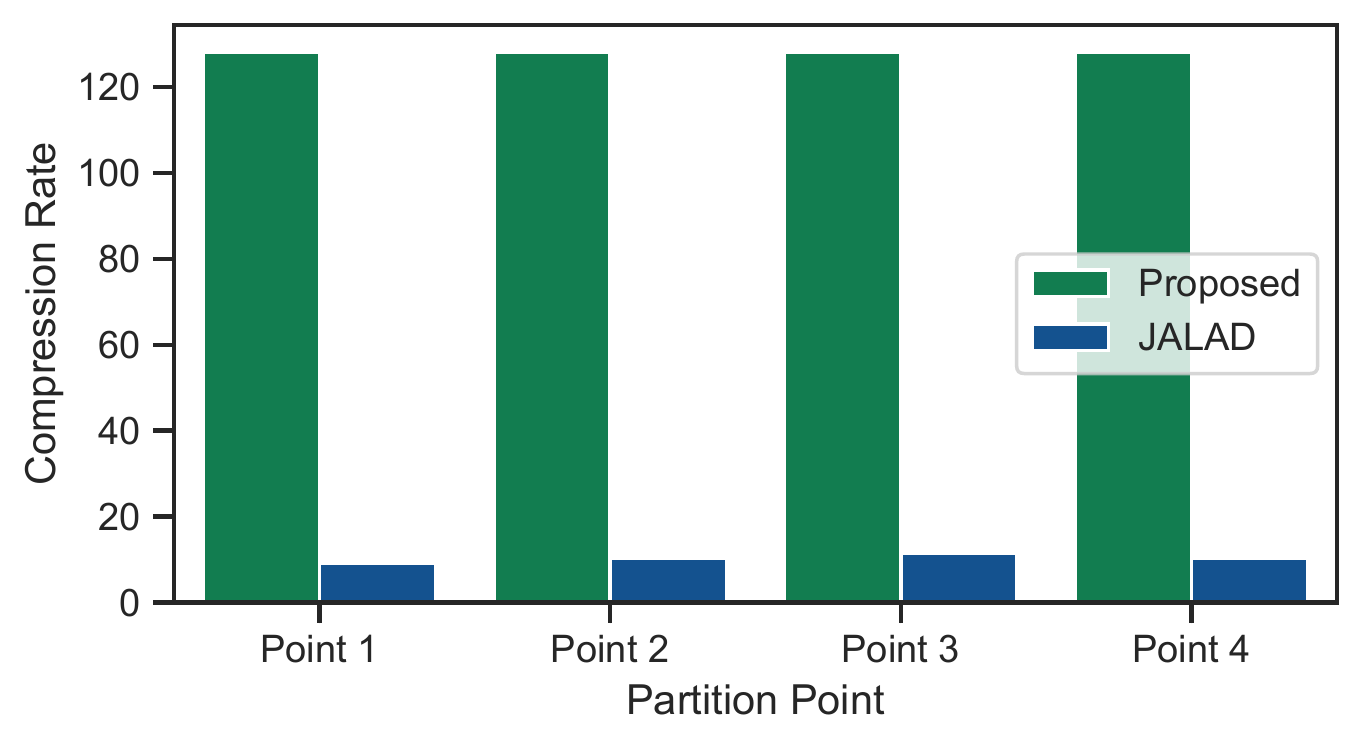}
    \label{fig:ae_main_v}}
  \hfil
  \subfloat[]{\includegraphics[width=0.475\linewidth]{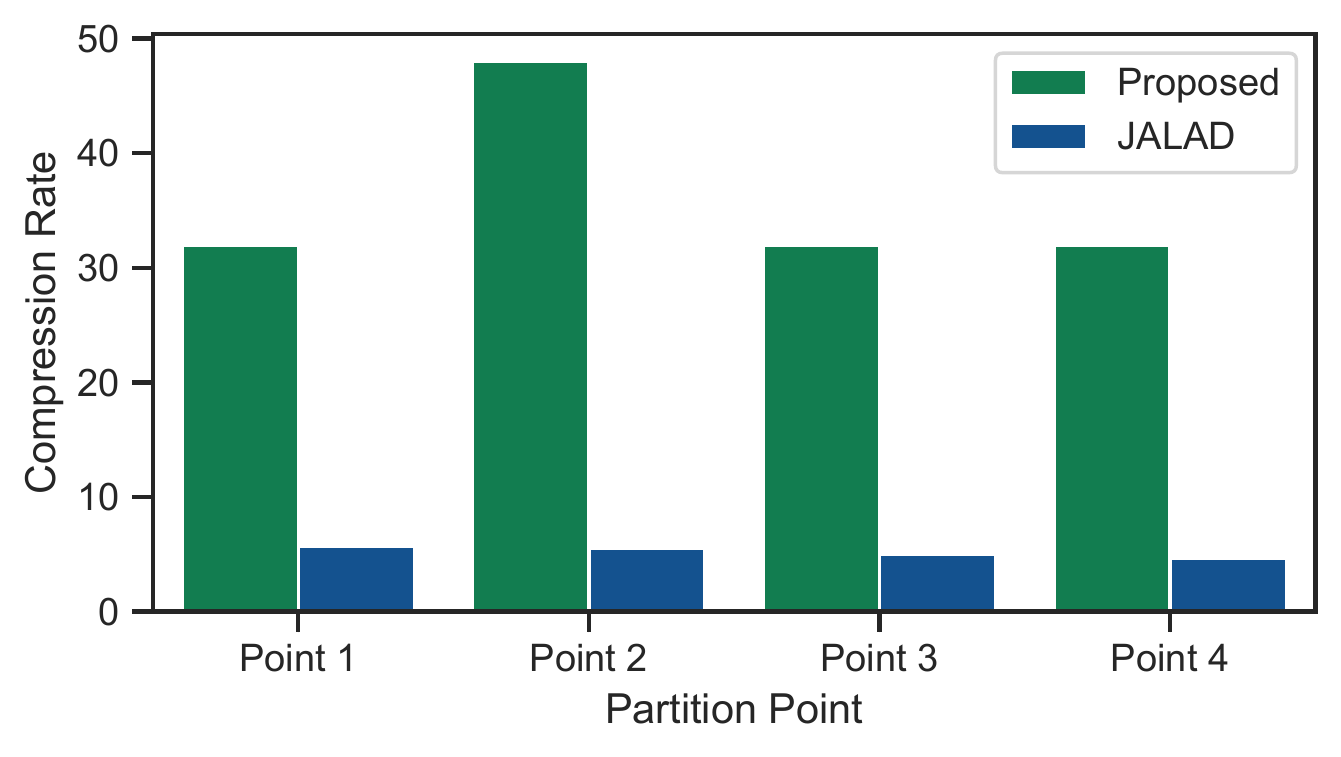}
    \label{fig:ae_main_m}}
  \vfil
  \subfloat[]{\includegraphics[width=0.475\linewidth]{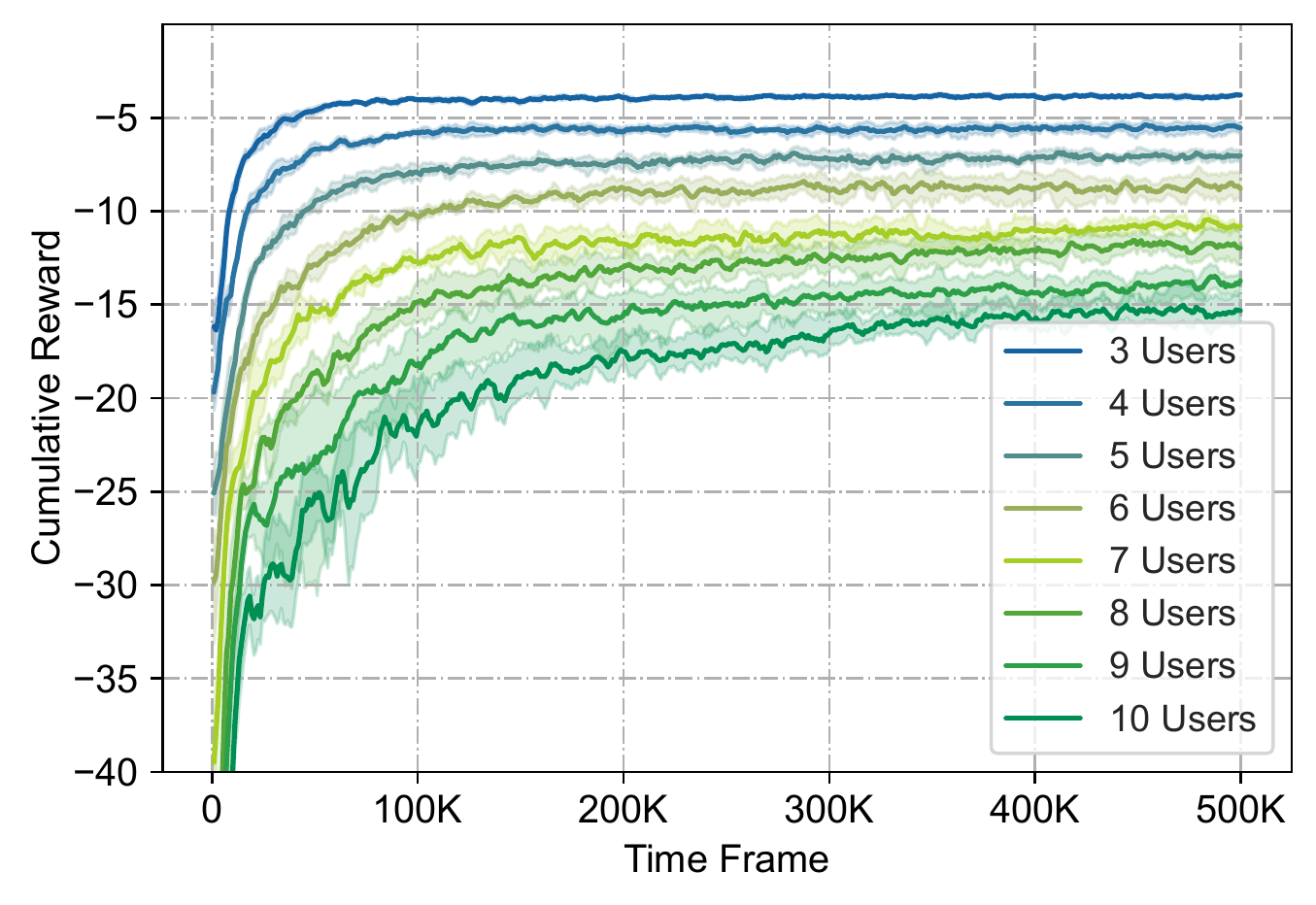}
    \label{fig:user_num_v}}
  \hfil
  \subfloat[]{\includegraphics[width=0.475\linewidth]{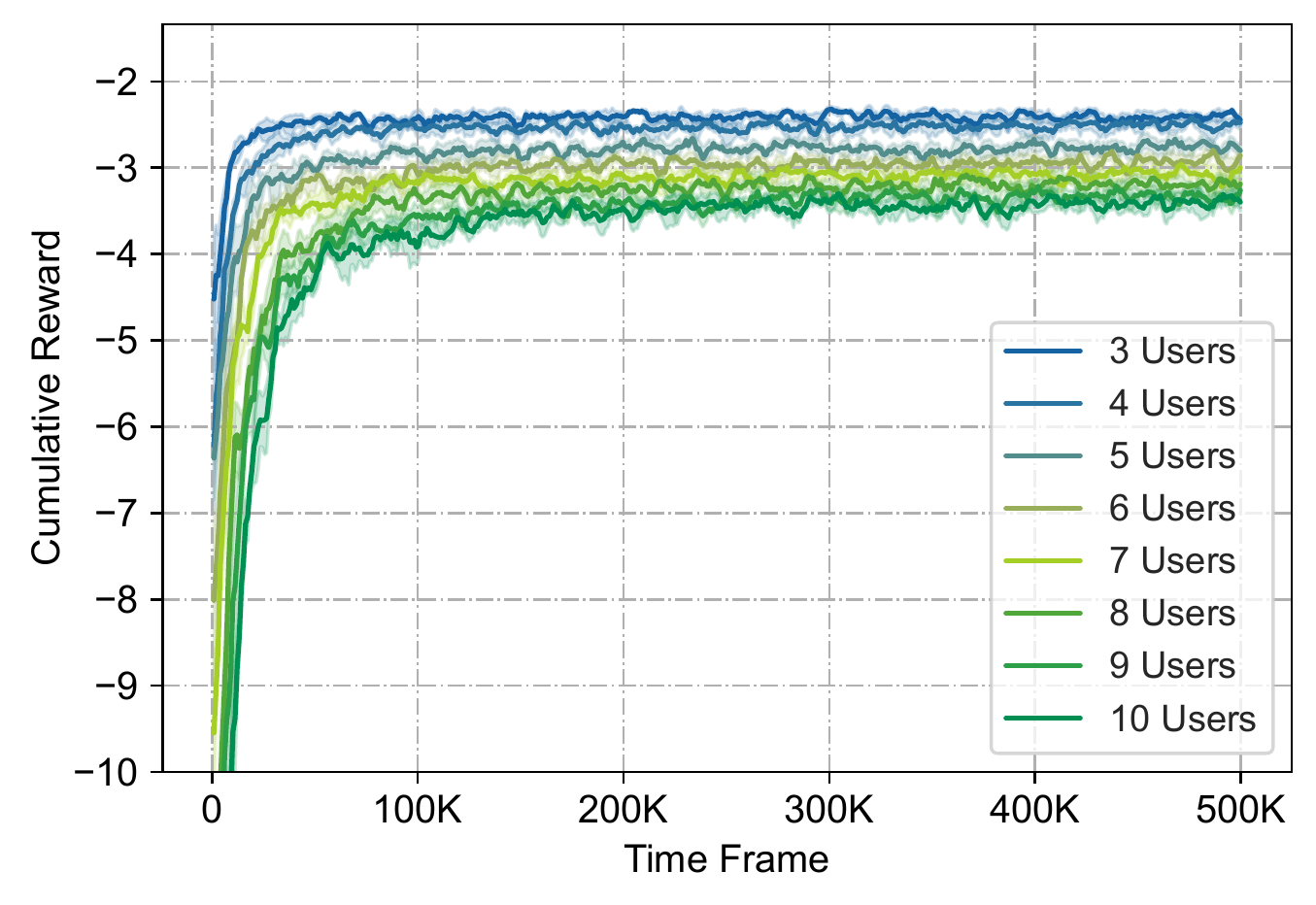}
    \label{fig:user_num_m}}
  \vfil
  \subfloat[]{\includegraphics[width=0.475\linewidth]{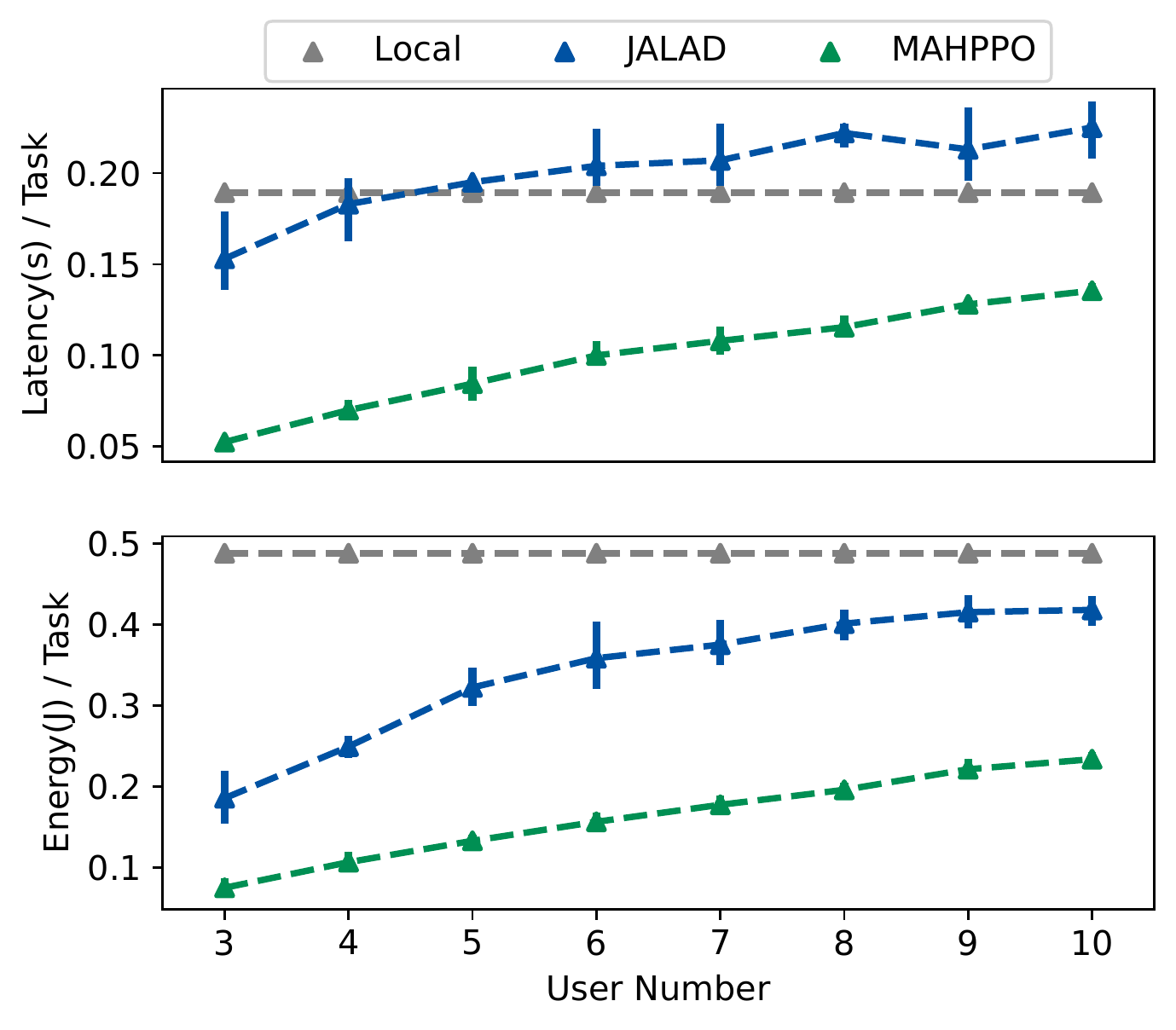}
    \label{fig:user_num_overhead_v}}
  \hfil
  \subfloat[]{\includegraphics[width=0.475\linewidth]{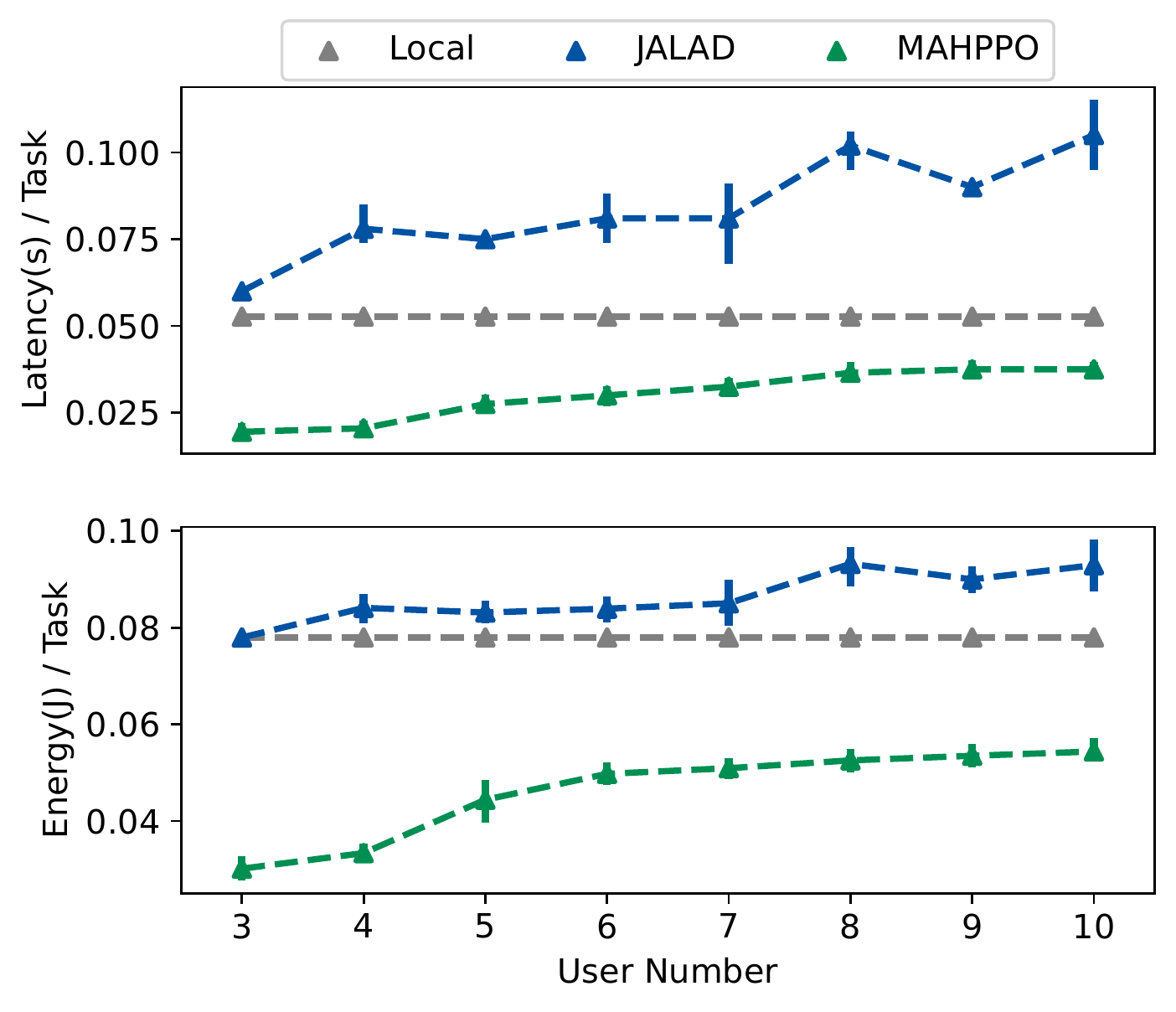}
    \label{fig:user_num_overhead_m}}
  \caption{Results with more network architectures: intermediate feature compression rate on (a) VGG11 and (b) MobileNetV2; convergence performance of different UE number settings on (c) VGG11 and (d) MobileNetV2; averaged inference overhead of different UE number settings on (e) VGG11 and (f) MobileNetV2.}
  \label{fig:supplementary}
\end{figure*}

To further verify the effectiveness of the proposed MAHPPO algorithm, we test it on two other network architectures: VGG11 \cite{DBLP:journals/corr/SimonyanZ14a} and MobileNetV2 \cite{DBLP:conf/cvpr/SandlerHZZC18}.
We evaluate their convergence performance and overhead saving performance under different UE number settings.

For VGG11, we select 4 partitioning points after MaxPool layers.
For MobileNetV2, we select 4 partitioning points after the last batch normalization layer of residual blocks containing a downsampling layer.
Other settings are the same as previous experiment setups.
Fig. \ref{fig:supplementary} reports the results.
The convergence performance of both networks is similar to that of ResNet18, and the overhead saving performance of the MAHPPO algorithm is also similar.
However, JALAD is better than local inference on VGG11, especially for energy consumption.
This is because the inference overhead of VGG11 is high, which makes the overhead of entropy coding ignorable.
In this case, JALAD is an efficient intermediate feature compression approach, and thus achieves better performance than local inference.
The results of either network demonstrate the effectiveness of the MAHPPO algorithm.

\section{Related Works}
\label{sec:related_works}

\subsection{Computation offloading}
Edge devices suffer from low computation resources and limited energy storage, which constrains the deployment of powerful applications.
To remedy this drawback, offloading a part of tasks to an external processor, termed computation offloading, has become popular.
Some early works offloaded tasks to the cloud.
For example, Zhang \textit{et al.} \cite{DBLP:journals/twc/ZhangWW15} investigated a mobile cloud system with a stochastic wireless channel and minimized the energy consumption of the mobile through offloading.
However, offloading to the cloud usually incurs significant latency in communication and lacks flexibility.
To address these issues, MEC places edge servers on the BS to serve users, providing a promising edge computing paradigm for quick and flexible applications and services. 
Chen \textit{et al.} \cite{DBLP:journals/ton/ChenJLF16} achieved efficient computation offloading in a multi-user multi-channel MEC scenario via game theory.
Guo \textit{et al.} \cite{DBLP:journals/ton/GuoZJLL18} integrated MEC into small cell networks (SCN).
They used genetic algorithm (GA) and particle swarm optimization (PSO) to find the optimal offloading decision.
Optimization problems in MEC are usually NP-hard, since both offloading channel selection and computation resources allocation require solving integer linear programming (ILP) problems.
Traditional approaches for solving such problems suffer from high computational complexity and slow convergence \cite{DBLP:journals/corr/abs-2001-09223}, which limits their application in the real world.

To this end, DRL provides a promising solution, and some DRL-based algorithms have been proposed for MEC systems.
Huang \textit{et al.} \cite{HUANG201910} studied the scenario where multiple UEs have multiple tasks to be offloaded to the edge server.
They adopted DQL to make task offloading and resource allocation decisions to minimize the overall overhead.
Similarly, some other previous works also adopted DQL to assist offloading \cite{DBLP:conf/wcnc/Li0LL18, 8657791, DBLP:journals/iotj/ZhangZLHMZ19}.
Chen \textit{et al.} \cite{Chen18IOT} used Double-DQL to obtain an optimal computation offloading policy.
Recently, some researchers used the DDPG algorithm to help MEC.
Chen \textit{et al.} \cite{DBLP:journals/ejwcn/ChenW20} adopted DDPG to minimize the long-term average energy consumption and latency at each user in the scenario with stochastic wireless channels and task arrival.
Their work ignored the latency at MEC servers.
Nath \textit{et al.} \cite{DBLP:conf/icc/NathLWF20} discretized the resource of the MEC server and trained a DRL agent via DDPG for offloading channel selection and MEC server resource allocation.
Xie \textit{et al.} \cite{DBLP:journals/iotj/XieTLYH21} formulated the optimization task as a partially observable MDP to model the fast time-varying wireless channel and proposed a deep recurrent Q-network to find the optimal offloading solution.

\subsection{Collaborative Inference}
Collaborative inference splits a DNN into two or more parts and executes each part on different devices.
Collaborative inference emphasizes the deployment of DNNs, while MEC focuses on the offloading process.
The split of DNNs can be conducted either horizontally or vertically.
We investigate the vertical split here, \textit{i.e.}, the split direction and the forward propagate direction are vertical.
Previous works \cite{DBLP:conf/icassp/HauswaldMZDCM14,DBLP:journals/tii/HuNQZL17} extracting features at local to assist the classification in the cloud establish prototypes of collaborative inference.
Kang \textit{et al.} \cite{DBLP:conf/asplos/KangHGRMMT17} proposed the collaborative inference paradigm, where the split can be conducted at any layer in a DNN, not only the last one.
They evaluated the overhead of DNN inference on the edge and cloud, and found that a favorable split decision reduces inference latency and saves energy significantly.
He \textit{et al.} \cite{DBLP:journals/iotj/HeG0QQ20} proposed to offload DNN inference tasks to multiple edge servers.
They formulated an MINLP problem and decomposed it to obtain partitioning deployment and resource allocation decisions.
To further speed up the inference, Li \textit{et al.} \cite{DBLP:conf/icann/LiLWDZF18} quantized weights of the front DNN part deployed on the UE, and the full-precision inference is adopted on the cloud.
Some DNN architectures are not linear, \textit{i.e.}, there exist cross-layer connections, such as DenseNet \cite{DBLP:conf/cvpr/HuangLMW17}, where a single point cannot achieve the partition.
To bridge this gap, Hu \textit{et al.} \cite{DBLP:conf/infocom/HuBWL19} depicted DNNs by direct acyclic graphs and adopted graphic algorithms to find the partitioning strategy. 
Yousefpour \textit{et al.} \cite{DBLP:journals/corr/abs-1909-00995} studied the failure of physical nodes when distributing DNN over multiple devices.
They introduced hyper connections between multiple devices to achieve robust collaborative inference. 
There are very few works combining MEC with collaborative inference to our best knowledge. 
Yang \textit{et al.} \cite{DBLP:conf/racs/YangLLMW19} studied a multi-agent MEC scenario with collaborative inference tasks.
However, they assumed a fixed uplink rate, ignoring the interference in the wireless channel.

There are also works compressing intermediate features at partitioning points to reduce the offloading latency.
For example, Shi \textit{et al.} \cite{DBLP:conf/infocom/ShiHZNZG19} pruned channels of a DNN, so the size of feature maps at partitioning points is reduced.
Li \textit{et al.} \cite{DBLP:conf/icpads/LiHJWWZ18} adopted quantization and Huffman coding to achieve feature compression.
Some works also designed convolution autoencoders to compress intermediate features.
Eshratifar \textit{et al.} \cite{DBLP:conf/islped/EshratifarEP19} was the first to propose to use autoencoders to compress features.
They compressed the features with depthwise separable convolution \cite{DBLP:conf/cvpr/Chollet17} and then the output with the JPEG algorithm.
Jankowski \textit{et al.} \cite{DBLP:conf/spawc/JankowskiGM20} adopted a pure DNN architecture, where both the encoder and the decoder consist of a convolution layer, a batch normalization layer, and an activation function.
However, these autoencoder compressors are complex, bringing significant additional overhead in inference.
There are works adopting other autoencoders \cite{DBLP:conf/icc/ShaoZ20, DBLP:conf/sensys/Yao0LWLSA20}, but these works also suffer from the complex architecture.

To tackle the above problem, we combine MEC with collaborative inference and assume a real-world communication environment.
Also, we design a lightweight autoencoder to reduce the size of offloaded data with little extra latency and energy consumption.

\section{Conclusion}
\label{sec:conclusion}

In this paper, we proposed a novel framework to achieve efficient multi-agent collaborative inference, where we first proposed an autoencoder-based intermediate feature compression method to enable flexible partitioning point selection.
Then we formulated the problem, which was then redefined to facilitate optimization.
Finally, we proposed an MAHPPO algorithm, which was composed of multiple actor networks and one global critic network, to solve the optimization problem.
We conducted extensive experiments to verify the effectiveness of the proposed framework.
The results showed that our framework can remarkably reduce the inference latency and energy consumption in multi-agent collaborative inference scenario.

\ifCLASSOPTIONcaptionsoff
  \newpage
\fi

\bibliographystyle{IEEEtran}
\bibliography{ref.bib}

\vfill

\end{document}